\newif\ifcolt
\coltfalse

\ifcolt
\documentclass[anon,12pt]{colt2021} 
\else
\documentclass{article}
\fi

\usepackage{amsmath}

\usepackage{amssymb}

\ifcolt
 \usepackage[algo2e]{algorithm2e}
 \usepackage{caption}
\else
 \usepackage{algcompatible} %
\usepackage{algorithm}     %
 \usepackage{algorithmicx}  %
 \usepackage{algpseudocode}
\usepackage{lipsum}
\usepackage{algpseudocode}
\usepackage{caption}
\usepackage{xspace}

\fi

\usepackage[normalem]{ulem}

\usepackage{tikz}
\usetikzlibrary{calc,positioning, math, tikzmark}
\usepackage{pgfplots}

\ifcolt

\else
\usepackage{amsthm}
\usepackage[colorlinks, citecolor=blue]{hyperref}
\usepackage[left=3cm,right=3cm,bottom=4cm]{geometry}
\fi

\usepackage{bbm}
\usepackage{bm}

\usepackage{comment}

\usepackage{wrapfig}

\newcommand{\alg}{{\texttt{FilTrun}}\xspace}

\newif\ifmylinenumbers
\mylinenumbersfalse

\ifmylinenumbers

\usepackage[mathlines]{lineno}
\linenumbers

\newcommand*\patchAmsMathEnvironmentForLineno[1]{%
   \expandafter\let\csname old#1\expandafter\endcsname\csname #1\endcsname
   \expandafter\let\csname oldend#1\expandafter\endcsname\csname end#1\endcsname
   \renewenvironment{#1}%
      {\linenomath\csname old#1\endcsname}%
      {\csname oldend#1\endcsname\endlinenomath}}%
\newcommand*\patchBothAmsMathEnvironmentsForLineno[1]{%
   \patchAmsMathEnvironmentForLineno{#1}%
   \patchAmsMathEnvironmentForLineno{#1*}}%
\AtBeginDocument{%
  \patchBothAmsMathEnvironmentsForLineno{equation}%
\patchBothAmsMathEnvironmentsForLineno{align}%
\patchBothAmsMathEnvironmentsForLineno{flalign}%
\patchBothAmsMathEnvironmentsForLineno{alignat}%
\patchBothAmsMathEnvironmentsForLineno{gather}%
\patchBothAmsMathEnvironmentsForLineno{multline}%
}

\newcounter{EditBlock}
\newcounter{CondEditBlock}

\newcommand{\editstart}{\renewcommand{\linenumberfont}{\normalfont\sffamily\footnotesize\bf\color{blue}}\stepcounter{EditBlock}\linelabel{L:editstart_\theEditBlock}}
\newcommand{\editfinish}{\renewcommand{\linenumberfont}{\normalfont\sffamily\tiny\color{black}}\linelabel{L:editfinish_\theEditBlock}}

\else
\newcommand{\editstart}{}
\newcommand{\editfinish}{}
\fi
\newcommand{\ev}[1]{\mathbb{E} \left [ #1 \right ] }
\newcommand{\evwrt}[2]{\mathbb{E}_{#1} \left [ #2 \right ] }
\newcommand{\pr}[1]{\mathbb{P} \left ( #1 \right ) }

\newcommand{\snorm}[1]{\Vert #1 \Vert}

\newcommand{\one}[1]{\mathbbm{1} \left [ #1 \right ]}

\newcommand{\reals}{\mathbb{R}}

\newtheorem{lem}{Lemma}
\newtheorem*{lem*}{Lemma}
\newtheorem{thm}{Theorem}
\ifcolt

\else

\fi

\ifcolt
\else
\newtheorem{example}{Example}
\fi

\newtheorem{rem}{Remark}
\newtheorem{cor}{Corollary}

\newcommand{\tSigma}{\widetilde{\Sigma}}

\newcommand{\tmean}{\mathsf{TMean}}
\newcommand{\tsum}{\mathsf{TSum}}
\newcommand{\mean}{\mathsf{Mean}}

\newcommand{\vx}{\vec{x}}

\newcommand{\vZ}{\vec{Z}}

\newcommand{\vy}{\vec{y}}
\newcommand{\vz}{\vec{z}}
\newcommand{\va}{\vec{a}}
\newcommand{\vu}{\vec{u}}

\newcommand{\vw}{\vec{w}}
\newcommand{\vmu}{\vec{\mu}}
\newcommand{\vmup}{\vec{\mu}'}
\newcommand{\mup}{{\mu'}}
\newcommand{\Sigmap}{{\Sigma'}}
\newcommand{\vmud}{\vec{\mu}^{(d)}}
\newcommand{\Sigmad}{\Sigma^{(d)}}
\newcommand{\vxp}{\vec{x}'}
\newcommand{\vxpp}{\vec{x}''}
\newcommand{\vnu}{\vec{\nu}}
\newcommand{\vnud}{\vec{\nu}^{(d)}}
\newcommand{\nud}{\nu^{(d)}}

\newcommand{\lambdad}{\lambda^{(d)}}

\let\vec\bm

\newcommand{\funcd}{\Psi_d}
\newcommand{\funcinf}{\Psi_\infty}
\newcommand{\funcinfinv}{\Psi^{-1}_\infty}

\newcommand{\pp}[1]{{#1}'}
\newcommand{\xp}{x'}

\newcommand{\mN}{\mathcal{N}}

\newcommand{\mB}{\mathcal{B}}
\newcommand{\mC}{\mathcal{C}}
\newcommand{\mCp}{\mathcal{C}'}
\newcommand{\mD}{\mathcal{D}}
\newcommand{\mDp}{\mathcal{D}'}
\newcommand{\tmC}{\widetilde{\mathcal{C}}}

\newcommand{\loss}{\mathcal{L}}
\newcommand{\optloss}{\mathcal{L}^\ast}

\newcommand{\sgn}{\text{sgn}}

\newcommand{\adv}{\mathsf{Adv}}

\ifcolt
\else

\algnewcommand\algorithmicinput{\textbf{Input:}}     %
\algnewcommand\INPUT{\item[\algorithmicinput]}       %
\algnewcommand\algorithmicoutput{\textbf{Output:}}   %
\algnewcommand\OUTPUT{\item[\algorithmicoutput]}     %
\algrenewcommand\algorithmicrequire{\textbf{Input:}} %
\algrenewcommand\algorithmicensure{\textbf{Output:}} %

\fi

\title{Robust Classification Under $\ell_0$ Attack for the Gaussian Mixture
  Model}
\author{Payam Delgosha\thanks{Department of Computer Science, University of
    Illinois at Urbana Champaign, IL, \texttt{delgosha@illinois.edu}}
  \qquad Hamed Hassani\thanks{Department of Electrical and Systems Engineering,
    University of Pennsylvania, Philadelphia, PA, \texttt{hassani@seas.upenn.edu}}
  \qquad Ramtin Pedarsani\thanks{Department of Electrical and Computer
    Engineering, University of California, Santa Barbara, Santa Barbara, CA, \texttt{ramtin@ece.ucsb.edu}}}
\begin{document}
\maketitle


\begin{abstract}
It is well-known that machine learning models are vulnerable to small but cleverly-designed adversarial perturbations that can cause misclassification. While there has been major progress in designing attacks and defenses for various adversarial settings, many fundamental and theoretical problems are yet to be resolved.  In this paper, we consider classification in the presence of $\ell_0$-bounded adversarial perturbations, a.k.a. sparse attacks. This setting is significantly different from other $\ell_p$-adversarial settings, with $p\geq 1$, as the $\ell_0$-ball is non-convex and highly non-smooth.  
Under the assumption that data is distributed according to the Gaussian mixture model, our goal is to characterize the optimal robust classifier and the corresponding robust classification error as well as a variety of trade-offs between robustness, accuracy, and the adversary's budget. To this end, we develop a novel classification algorithm called $\alg$ that has two main modules: Filtration and Truncation. The key idea of our method is to first filter out the \emph{non-robust} coordinates of the input and then apply a carefully-designed truncated inner product for classification. By analyzing the performance of $\alg$, we derive an upper bound on the optimal robust classification error. We further find a lower bound by designing a specific adversarial strategy that enables us to derive the corresponding robust classifier and its achieved error. For the case that the covariance matrix of the Gaussian mixtures is diagonal, we show that as the input's dimension gets large, the upper and lower bounds converge; i.e. we characterize the asymptotically-optimal robust classifier. Throughout, we discuss several examples that illustrate interesting behaviors such as the existence of a \emph{phase transition} for adversary's budget determining whether the effect of adversarial perturbation can be fully neutralized or not.
\end{abstract}

\section{Introduction}
\label{sec:intro}
Machine learning  has been widely used in a variety of applications including image recognition, virtual assistants, autonomous driving, many of which are safety-critical. Adversarial attacks to machine learning models in the form of a small perturbation added to the input have been shown to be effective in causing classification errors~\cite{biggio2013evasion,szegedy,goodfellow2014explaining,carlini2017,madry2017towards}. 
Formally, the adversary aims to perturb the data in a small $\ell_p$-neighborhood so that the perturbed data is ``close'' to the original data (e.g. imperceptible perturbation in the case of an image) and misclassification occurs. 
There have been a variety of attacks and defenses proposed in the literature which mostly focus on $\ell_2$ or $\ell_\infty$ bounded perturbations~\cite{carlini2018,marzi,wong}. The state-of-the-art empirical defense 
against adversarial attacks is iterative training with adversarial examples~\cite{Madry_ICLR}. While adversarial training can improve robustness, it is shown that there is
a fundamental tradeoff between robustness and test accuracy, and such defenses typically lack good generalization performance~\cite{tsipras2019robustness,su2018robustness,raghunathan2019adversarial,al2019fundamental,zhang2019theoretically, javanmard2020precise}. 

The focus of this paper is different from such prior work as we consider the 
problem of robust classification under $\ell_0$-bounded attacks.  
In this
setting, given a pre-specified budget $k$, the adversary can choose up to $k$
coordinates and arbitrarily change the value of the input at those coordinates.
In other words, the adversary can change the input within the so-called
$\ell_0$-ball of  radius $k$. 
In contrast with $\ell_p$-balls \textcolor{black}{($p \geq 1$)}, the $\ell_0$-ball
is non-convex and highly non-smooth.
Moreover, the $\ell_0$ ball
contains inherent discrete (combinatorial) structures that can be exploited by
both the learner and the adversary. As a result, the $\ell_0$-adversarial
setting bears several fundamental challenges that are absent in other
adversarial settings commonly studied in the literature and most techniques from
prior work do not readily apply in the $\ell_0$ setting.  Complicating matters
further, it can be shown that any piece-wise linear classifier, e.g. a
feed-forward deep neural  network with ReLu activations, completely fails in the
$\ell_0$ setting \cite{shamir2019simple}. These all point to the fact that new methodologies are
required in the $\ell_0$  setting.

The $\ell_0$-adversarial setting involves sparse attacks that perturb only a
small portion of the input signal. This has a variety of  applications
{\color{black}including} natural
language processing~\cite{jin2019bert},  malware
detection~\cite{grosse2016adversarial}, and physical attacks in object  detection~\cite{li2019adversarial}. 
Prior work on $\ell_0$ adversarial attacks can be divided into two categories of white-box attacks that are gradient-based,
e.g.~\cite{carlini2017,papernot2016limitations,modas2019sparsefool}, and black-box attacks 
based on zeroth-order optimization,
e.g.~\cite{schott2018towards,croce2020sparse}. Defense strategies against
$\ell_0$-bounded attacks have also been proposed, e.g. defenses based on
randomized ablation~\cite{levine2020robustness} and defensive
distillation~\cite{papernot2016distillation}. Moreover, \cite{shamir2019simple}
develops a simple mathematical framework to show the existence of targeted
adversarial examples with $\ell_0$-bounded perturbation in arbitrarily
deep
neural  networks. 

Despite this interesting recent progress and practical relevance, many fundamental theoretical questions in the $\ell_0$-setting have so far been unanswered: \emph{What are the key properties of a robust classifier (recall that all piece-wise linear classifiers fail)? What is the optimal robust classifier in standard theoretical settings such the Gaussian mixture model for data?  Is there a trade-off between robustness and accuracy?  How does the (optimal) robust classification error behave as the adversary's budget $k$ increases? Are there any phase~transitions?}

We consider the problem of classification with $\ell_0$-adversarially perturbed inputs under the assumption that data is distributed according to the Gaussian mixture model. We formally introduce this setting in Section~\ref{sec:problem-form}, and address the questions above 
in the proceeding sections. In particular, instead of searching for the exact
form of the optimal robust classifier {\color{black}(which is intractable)}, we follow a
design-based approach: We introduce~a novel algorithm for classification as well
as strategies for the adversary. We then precisely characterize the error
performance of these methodologies, and consequently, analyse the optimal robust
classification error, tradeoffs between robustness and accuracy, phase
transitions, etc.  We envision that our proposed classification method introduces
important modules and insights that are necessary to obtain robustness against
$\ell_0$-adversaries for general data distributions (and practical datasets),
going beyond the theoretical setting of this  paper.    
\vspace{0.1cm}

\noindent\textbf{Summary of Contributions.} The main contributions of this paper are as follows:
\vspace{-0.12cm}
\begin{itemize}
    \item We propose a new robust classification algorithm called $\alg$ that is based on two main modules: \textbf{Fil}tration and \textbf{Trun}cation (See Section~\ref{sec:alg} and Algorithm~\ref{alg:l0-robust-class} therein). The filtration module 
   removes the \emph{non-robust} coordinates (features) from the input by
   zeroing out their values. The result is then passed through the truncation
   module which returns a label by computing a \emph{truncated inner product}
   with a weight vector whose weights are optimized according to the
   distribution of un-filtered
   (surviving)
   coordinates.  The truncation module  is inspired
   by tools from robust statistics and guarantees that major outlier values in
   the input vector, which are possibly caused by the adversary, do not pass to
   affect the final decision. We highlight that 
   the proposed classifier is highly nonlinear.
   This is consistent with the simple observation that any
   linear classifier fails to be robust in the presence of $\ell_0$  attacks. 
    

    \vspace{-0.12cm}
    \item We analytically derive the robust
        classification error  of the proposed classifier. {\color{black}This in
          particular} serves as an
      upper bound on the optimal robust classification
         error (See Theorem~\ref{thm:upper-bound} and Corollary~\ref{cor:upper-bound-diagonal}).
         \vspace{-0.12cm}
    \item We introduce adversarial strategies which, given sufficient budget,
      perturb the input in a way that the information about the true label is
      totally erased {\color{black}within the adversarially modified coordinates}.  The key idea is to pick a subset of the
        coordinates and to modify their distribution so that they become
        independent from the true label. This leads to a
        lower bound for the optimal
      robust error. (See
      Theorems~\ref{thm:lower-bound-diag} and \ref{thm:general-lower-bound}).
      \vspace{-0.12cm}
    \item In the case of having a diagonal covariance matrix for the Gaussian mixtures, we prove
      that our proposed algorithm $\alg$ is indeed
      \emph{asymptotically-optimal}, i.e. as the input dimension $d$ approaches infinity,
      the upper and lower bounds converge to the same analytical expression
      (See Theorems~\ref{thm:ind-bound-finite}  in Section~\ref{sec:asymptotic-diagonal}).
      To the best of our knowledge, this is the first result that establishes optimality for the robust classification error of any mathematical model with $\ell_0$ attack.
      \vspace{-0.12cm}
    \item We discuss our
        results through several example scenarios. In certain scenarios, a
        phase transition is observed in the sense that for a threshold
        $\alpha_0$, when the adversary's  budget is asymptotically below
        $d^{\alpha_0}$, its effect can be completely neutralized, while if the
        adversary's budget is above $d^{\alpha_0}$, no classifier can do better
        than a naive classifier. In some other scenarios, no sharp phase transition is existent, leading to a trade-off between robustness and accuracy. 
\end{itemize}




\editstart
\vspace{-.07cm}
\section{Problem Formulation} \label{sec:problem-form}

We consider the binary Gaussian mixture model where the distribution for the
data generation is specified by the label being  $y \sim \text{Unif}\{\pm 1\}$
and $\vx \sim \mN(y \vmu, \Sigma)$, i.e. the Gaussian distribution with mean $y
\vmu$ and covariance  matrix
$\Sigma$, where $\vmu \in \reals^d$ and $\Sigma$ is positive definite.
Hereafter, we denote this distribution by $(\vx,y) \sim \mathcal{D}$ and refer
to $y$ as the label and to $\vx$ as the input. Our results correspond to
arbitrary choices of $\vmu$ and $\Sigma$, however, we consider as running
example an important special case in which $\Sigma$ is a
\textcolor{black}{diagonal} matrix, i.e.\ the coordinates of $\vx$ are
independent conditioned on $y$.  Focusing on classification, we consider
functions of the form  $\mC: \reals^d \rightarrow \{-1,1\}$ that predict the
label from the input. As a metric for the discrepancy between the 
prediction of the classifier on the input $\vx$ and the true label $y$, we
consider the 0-1 loss 
$\ell(\mC; \vx, y) = \one{\mC(\vx) \neq y }.$ 
We consider classification in the presence of an adversary that  perturbs the input $\vx$ within the $\ell_0$-ball of radius $k$:
$$ \mathcal{B}_0(\vx,k) := \{\vx' \in \reals^d: \snorm{\vx - \vxp}_0 \leq k\},$$
where for $\vx = (x_1, \cdots, x_d)$ we define $\snorm{\vx}_0 := \sum_{i=1}^d \one{x_i \neq 0}$.
In other words,  the adversary can arbitrarily modify at most $k$ coordinates of $\vx$ to obtain
$\vxp$, and feed the new vector $\vxp$  to the classifier.  We call
$k$ the \emph{budget} of the adversary. In this setting, the
{\color{black}robust classification error} of a classifier
  $\mathcal{C}$ is {\color{black}defined to be} the following:
\begin{equation}
  \label{eq:general-l0-loss-def}
  \loss_{\vmu, \Sigma}(\mC, k) := \mathbb{E}_{(\vx,y) \sim \mathcal{D}}\,\biggl[\max_{\vx' \in \mathcal{B}_0(\vx,k)}\,\ell(\mC;\vxp,y)\,\biggr]. 
\end{equation}
We aim to design classfiers with minimum robust classification  error. Hence, we define the
\emph{optimal robust {\color{black}classification} error} by minimizing \eqref{eq:general-l0-loss-def} over all possible classifiers:
\begin{equation}
  \label{eq:general-opt-loss-def}
  \optloss_{\vmu, \Sigma}(k):= \inf_{\mC} \loss_{\vmu, \Sigma}(\mC, k).
\end{equation}
Our goal in this paper is to precisely characterize $\optloss_{\vmu, \Sigma}(k)$
{\color{black}parameterized by} $\Sigma, \vmu$ and in different regimes of the adversary's budget $k$. 

 It is well known that in the absence of the adversary, i.e.\ when
  $k=0$, the Bayes optimal classifier is the linear classifier $\mathcal{C}(\vx)
  = \text{sgn} \left(\langle \Sigma^{-1} \vmu , \vx \rangle \right)$ which
  achieves the  \emph{optimal standard  error} 
  of $\bar{\Phi}(\snorm{\vnu}_2)$
  where $\vnu := \Sigma^{-1/2} \vmu$ {\color{black}and $\bar{\Phi}(x) := 1 -
    \Phi(x)$ denotes the complementary CDF of a standard normal distribution}.
  In order to fix the baseline, specifically  to
  have a meaningful asymptotic discussion, we may assume without loss of
  generality  that
  \begin{equation}
    \label{eq:norm-nu-is-one}
    \snorm{\vnu}_2 = \snorm{\Sigma^{-1/2} \vmu}_2 = 1.
  \end{equation}
 Hence, the optimal standard error, which is a lower bound for \eqref{eq:general-opt-loss-def}, becomes $\bar{\Phi}(1)$. 
 
 To highlight some of the main challenges of the $\ell_0$-adversarial setting, we
 note that linear classifiers in general have been very successful in the
 Gaussian mixture setting. Apart from the fact that the Bayes-optimal classier
 is linear (when there is no adversary),  even when the adversarial corruptions
 are chosen in a $\ell_p$-ball for $p \geq 1$ it can be shown that the optimal
 robust classifiers in many cases are also linear (see \cite{bhagoji,hassani}). In contrast,
 in the presence of $\ell_0$-adversaries, it is not hard to show that \emph{any}
 linear classifier completely fail. More precisely,  when $\mathcal{C}$ is linear and $k
 \geq 1$, we have $\loss_{\vmu, \Sigma}(\mC, k) = {\color{black}\frac{1}{2}}$. Such failure of linear
 classifiers showcases, on the one hand, how powerful the adversary is, and on
 the other hand, the necessity of new methodologies in designing robust  classifiers. 
 
 \vspace{.2cm}
 \noindent\textbf{Further Related Work.} For $\ell_p$ adversaries, $p \geq 1$, Gaussian mixture models have been the main setting used in  prior work to  investigate optimal rules, trade-offs, and various other phenomena for robust classification; See e.g. \cite{schmidt2018adversarially,bhagoji,hassani, hayes2020provable,richardson2020bayes, dan2020sharp, pydi2020adversarial, chen2020more, min2020curious, puranik2020adversarially}. Further, \cite{shafahi2018adversarial} considers data to be uniformly distributed on the sphere or cube and shows the inevitability of adversarial examples in $\ell_p$-settings, $p\geq 0$. In contrast, to the best of our knowledge, our work provides the first comprehensive study on the $\ell_0$-adversarial setting using the Gaussian mixture model. 

\vspace{.2cm}
\noindent \textbf{Notation.} 
Given two vectors $\vx, \vy \in \reals^d$, $\vx \odot \vy \in \reals^d$
denotes the elementwise product of $\vx$ and $\vy$, i.e. $(x_1 y_1, \dots, x_d
y_d)$. Moreover, $\text{sort}(\vx)$ denotes the vector containing the elements
in $\vx$ in descending order. For $a \in \reals$, $\text{sgn}(a)$ returns the
sign of $a$.
We use $[d]$ to denote the set $\{1, \dots, d\}$ and $[i:j]$ denotes
the set $\{i, i+1, \dots, j\}$. 
Given a vector
$\vx \in \reals^d$ and a subset $A \subseteq [d]$, $\vx_A = (x_a: a \in A) \in \reals^{|A|}$
denotes the subvector of $\vx$ consisting  of the coordinates in $A$. 
Given a  matrix
$\Sigma$, its diagonal part, denoted by  $\tSigma$, has the same diagonal
entries as $\Sigma$ and its other entries  are $0$. 
  Given a
matrix $A \in \reals^{d \times d}$, $\snorm{A}_\infty$ denotes the operator norm
of $A$ induced by the vector $\ell_\infty$ norm, i.e.\ $\snorm{A}_{\infty} := \sup_{\vx \neq 0} \snorm{A\vx}_\infty/\snorm{\vx}_\infty = \max_{1 \leq i \leq d} \sum_{j=1}^d |A_{i,j}|$.

\editfinish








\section{Main Results}
\label{sec:main-results}

\editstart


In this section, we state our main results that include (i) the proposed
algorithm and its performance analysis that serves as an upper bound on the
optimal robust classification error (Section \ref{sec:upper-bound}), (ii) lower
bound on the optimal robust classification error (Section
\ref{sec:lower-bound-ind}), and (iii) discussion on the optimality of the
proposed algorithm (Section \ref{sec:diagonal-compare-asymp}).  Throughout, we
illustrate our theoretical results and their ramifications via several examples. 

\subsection{Upper Bound on the Optimal Robust Classification Error:  Algorithm Description and Theoretical Guarantees}
\label{sec:upper-bound}

In Section~\ref{sec:alg}, we introduce  $\alg$, our proposed robust
classification algorithm, and in Section~\ref{sec:uppper-bound-result}, we
analyze its performance.



\subsubsection{Algorithm Description}
\label{sec:alg}
We describe our proposed algorithm $\alg$, a robust classifier which is based on two main modules: Truncation and Filtration. We first introduce each of these modules and then proceed with describing the classifier.  

\noindent\textbf{Truncation.} 
  Given vectors $\vw, \vx \in \reals^d$ and an integer $0 \leq k <d/2$, we
  define the
  $k$--truncated inner product of $\vw$ and $\vx$
  as the summation of the element-wise product of $\vw$ and $\vx$ after removing the
  top and bottom $k$ elements, and denote it 
 by $\langle  \vw, \vx \rangle_k$.
  More precisely, let $\vz := \vw \odot \vx \in \reals^d$ be the element-wise product of
  $\vw$ and $\vx$ and let $\textbf{s} = (s_1, \cdots,s_d) =  \text{sort}(\vz)$ be obtained by sorting
  coordinates of $\vz$ in  descending order. We~then~define
  \begin{equation}
    \label{eq:trunc-in-prod-def}
    \langle \vw, \vx \rangle_k := \sum_{i=k+1}^{d-k} s_i.
  \end{equation}

Note that when $k=0$, this reduces to the normal inner product $\langle  \vw,
\vx \rangle$. Truncation is a natural method to remove ``outliers'' which
might exist in the data due to an  adversary modifying some coordinates.
Therefore, we 
expect the truncated inner product  to be robust against $\ell_0$ perturbations.
The following lemma formalizes this. The proof of Lemma~\ref{lem:trun-ip-bound}
is given in Appendix~\ref{sec:trunc-stat}.

\begin{lem}
\label{lem:trun-ip-bound}
  Given $\vx, \vxp, \vw \in \reals^d$, for integer $k$ satisfying $\snorm{\vx - \vxp}_0 \leq k <
  d/2$, we have
  \begin{equation*}
    |\langle  \vw, \vxp \rangle_k - \langle \vw, \vx \rangle | \leq 8k \snorm{\vw \odot \vx}_\infty.
  \end{equation*}
\end{lem}

\ifcolt

\else
\begin{figure}
  \centering
  \include{fig_filtran}
  \vspace{-5mm}
    \caption{  \label{fig:filtran} Schematic of $\alg$.}
\end{figure}

\fi

In the context of our problem, this lemma suggests that if the budget of the
adversary
is at most $k$, we can bound the difference between  the $k$--truncated inner product between $\vw$ and
the adversarially modified sample $\vxp$ and  the 
(non-truncated) inner product between $\vw$ and the original sample $\vx$.
Recall that  in  the absence
of the adversary, the optimal Bayes classifier is a linear classifier of the form
$\sgn(\langle \vw, \vx \rangle)$ with $\vw = \Sigma^{-1} \vmu$.
Hence, motivated
by Lemma~\ref{lem:trun-ip-bound}, one can argue that $\sgn(\langle \vw, \vxp \rangle_k)$ would be
robust against $\ell_0$ adversarial attacks with budget at most $k$ assuming we can
appropriately control the bound of Lemma~\ref{lem:trun-ip-bound}. However, this is not enough--it
turns out that in certain cases, 
we need
to \emph{filter out} some of the input coordinates and perform the truncation on the
remaining coordinates, which we call \textcolor{black}{ the \emph{surviving}
  coordinates.} 

\noindent \textbf{Filtration} refers to discarding some of the coordinates of
the input. Intuitively, these coordinates are the \emph{non-robust} features
which do more harm than good when the input is adversarially corrupted.     More
precisely, given a fixed and nonempty subset of coordinates $F
\subseteq [d]$, we define the classifier $\mC^{(k)}_F$ as follows:

\begin{equation}
  \label{eq:robust-mC-F-def}
  \mC_F^{(k)}(\vxp) := \sgn \bigl(\, \langle  \vw (F), \vxp_F \rangle_k\, \bigr),
\end{equation}
where
\begin{equation*}
  \vw (F) := \Sigma_{F}^{-1} \vmu_{F},
\end{equation*}
and
\begin{equation}
  \label{eq:SigmaF}
 \Sigma_{F} = \evwrt{(\vx, y) \sim \mD}{(\vx_{F} -  \vmu_{F})(\vx_{F} -
   \vmu_{F})^T | y = 1}
\end{equation}
 is
the covariance matrix of $\vx_{F}$ conditioned on $y$, which is essentially the
submatrix of $\Sigma$ corresponding to the elements in $F$.  Note that $\vw(F)$ is the optimal Bayes
classifier of $y$ given $\vx_{F}$ in the absence of the adversary. It is easy
to see that  when
$\Sigma$ is diagonal, $\vw(F) = \vw_{F}$, but this might not hold in
general.

Algorithm~\ref{alg:l0-robust-class} and Figure~\ref{fig:filtran} illustrate the classification procedure $\alg$ given in \eqref{eq:robust-mC-F-def}.  So
far we have not explained how the set $F$ is chosen and the algorithm works with
any such set given as an input. 
\textcolor{black}{Later  we discuss how the set $F$ is 
chosen} (see Remarks~\ref{rem:general-F-choice-max} and~\ref{rem:diagonal-how-to-chose-F}).


\ifcolt
\else

\begin{algorithm}                                                                        %
\caption{$\alg$\label{alg:l0-robust-class}}                            %
  \begin{algorithmic}[1]                                                                 %
    \INPUT                                                                               %
    \Statex $k$: adversary's $\ell_0$ budget                                                %
    \Statex $\vmu, \Sigma$: parameters of the Gaussian distribution                      %
    \Statex $F$: the set of surviving coordinates                                           %
    \Statex $\vxp$: the corrupted input                                                                     %
    \OUTPUT                                                                              %
    \Statex $\mC^{(k)}_F(\vxp)$                                                %
    \Function{FilTrun}{$k, \vmu, \Sigma, F, \vxp$}                                   %
    \State \textbf{Filtering:} Construct $\vmu_F, \Sigma_F$ and $\vxp_F$
    corresponding to the coordinates in $F$
    \State Compute $\vw(F) \gets \Sigma_{F}^{-1} \vmu_{F}$
    \State   \textbf{Truncation:} Compute  $\langle \vw(F), \vxp_{F}
    \rangle_k$
    \State \textbf{Return} $\text{sgn} \left( \langle \vw(F), \vxp_{F}
  \rangle_k \right) $
    \EndFunction                                                                         %
  \end{algorithmic}                                                                      %
\end{algorithm}

\fi


\subsubsection{Upper bound on the robust classification error of \text{$\alg$}} 
\label{sec:uppper-bound-result}

Theorem~\ref{thm:upper-bound} below states an upper bound for the robust error associated with the classification algorithm \alg 
introduced in Section~\ref{sec:alg}. In particular, this yields an upper bound
on the optimal robust classification error. The proof of
Theorem~\ref{thm:upper-bound} is given in Appendix~\ref{sec:app-upper-bound-proof}.

\begin{thm}
\label{thm:upper-bound}
Assume that $\vmu, \Sigma$ are given such that~\eqref{eq:norm-nu-is-one} holds.
For a given nonempty $F \subseteq [d]$ and $0 \leq k < d /2$, we have
  \begin{equation}
\label{eq:thm-upper-bound-general-bond-error-CFk}
    \loss_{\vmu, \Sigma}(\mC_F^{(k)}, k) \leq \frac{1}{\sqrt{2 \log d }} +\bar{\Phi}\left( \snorm{\vnu(F)}_2 - \frac{16 k \sqrt{2\log d} \snorm{\tSigma_{F}^{1/2} \Sigma_{F}^{-1/2}}_\infty \snorm{\vnu(F)}_\infty}{\snorm{\vnu(F)}_2} \right),
  \end{equation}
  where $\Sigma_{F}$ is defined in~\eqref{eq:SigmaF},  $\tSigma_{F}$
  is the diagonal part of $\Sigma_{F}$, and 
  \begin{equation*}
    \vnu(F) := \Sigma_{F}^{-1/2} \vmu_{F}.
  \end{equation*}
  As a consequence, we obtain
  \begin{equation}
    \label{eq:upper-bound-general-max-F}
    \optloss_{\vmu, \Sigma}(k) \leq \frac{1}{\sqrt{2\log d}} + \min_{F \subseteq [d]} \bar{\Phi}  \left( \snorm{\vnu(F)}_2 - \frac{16 k \sqrt{2\log d} \snorm{\tSigma_{F}^{1/2} \Sigma_{F}^{-1/2}}_\infty \snorm{\vnu(F)}_\infty}{\snorm{\vnu(F)}_2} \right).
  \end{equation}
\end{thm}

\begin{rem}
Recall from Section~\ref{sec:alg} that $F$ is the set of coordinates used for classification (i.e. the information in the coordinates $F^c$ is discarded). Therefore, we essentially work with
$\vx_{F}$ as an input. If the adversary is not present, the optimal 
classification error is achieved via the Bayes linear classifier which has error
$\bar{\Phi}(\snorm{\vnu(F)}_2)$. However, due to the existence of an adversary,
we need to perform truncation which influences the error through the second term
inside the argument of $\bar{\Phi}$ in~\eqref{eq:thm-upper-bound-general-bond-error-CFk}.
\end{rem}

\begin{rem}
  \label{rem:general-F-choice-max}
  The bound in Theorem~\ref{thm:upper-bound}
  can be used as a guide to choose the set of surviving   coordinates $F$.
  More precisely, we can choose  $F$ which 
  \textcolor{black}{minimizes} the right hand side 
  in~\eqref{eq:upper-bound-general-max-F}. Later, in
  Section~\ref{sec:diagonal-compare-asymp}, we discuss a simpler mechanism for
  choosing $F$ when the covariance matrix $\Sigma$ is diagonal (see
  Remark~\ref{rem:diagonal-how-to-chose-F} therein).
\end{rem}

Here, we outline the proof of Theorem~\ref{thm:upper-bound}. Due to the
symmetry, we only need to analyze the classification error when $y=1$. In this
case, an error occurs only when there exists some $\vxp \in \mB_0(\vx, k)$ such
that $\langle \vw(F), \vxp_F\rangle_k \leq 0$. But since $\snorm{\vxp_F -
  \vx_F}_0 \leq \snorm{\vxp - \vx}_0 \leq k$,  Lemma~\ref{lem:trun-ip-bound}
implies that for such $\vxp$, we have $|\langle \vw(F), \vxp_F \rangle_k -
\langle \vw(F), \vx_F \rangle| \leq 8k \snorm{\vw(F) \odot \vx_F}_\infty$.
Therefore, the robust classification error is upper bounded by $\pr{\langle
  \vw(F), \vx_{F}  \rangle \leq 8k \snorm{\vw(F)\odot \vx_{F}}_\infty}$. But the
random variable $\langle \vw(F), \vx_{F}  \rangle$ is Gaussian with a  known
distribution, and the proof follows by bounding $\snorm{\vw(F)\odot
  \vx_{F}}_\infty$. See Appendix~\ref{sec:app-upper-bound-proof} for  details.

\editfinish

When the covariance matrix $\Sigma$ is diagonal,
$\Sigma_{F}$ is also diagonal and $\tSigma_{F}^{1/2} \Sigma^{-1/2}_{F} = I$.
Moreover, $\vnu(F) = \vnu_F$ where $\vnu := \Sigma^{-1/2} \vmu$. 
This yields the following corollary of Theorem~\ref{thm:upper-bound}.

\begin{cor}
  \label{cor:upper-bound-diagonal}
Assume that $\vmu, \Sigma$ are given such that~\eqref{eq:norm-nu-is-one} holds
and $\Sigma$ is diagonal. Then, for nonempty $F \subseteq [d]$
  we have
  \begin{equation*}
        \loss_{\vmu, \Sigma}(\mC_F^{(k)}, k) \leq \frac{1}{\sqrt{2 \log d }} +\bar{\Phi}\left( \snorm{\vnu_F}_2 - \frac{16 k \sqrt{2\log d}  \snorm{\vnu_F}_\infty}{\snorm{\vnu_F}_2} \right),
      \end{equation*}
      and in particular
        \begin{equation*}
    \optloss_{\vmu, \Sigma}(k) \leq \frac{1}{\sqrt{2\log d}} + \min_{F \subseteq [d]} \bar{\Phi}  \left( \snorm{\vnu_F}_2 - \frac{16 k \sqrt{2\log d}  \snorm{\vnu_F}_\infty}{\snorm{\vnu_F}_2} \right).
  \end{equation*}
\end{cor}

Now we discuss the above bounds via two examples, which we use 
as running examples to discuss our results in the subsequent sections as well.
In the following, $I_d \in \reals^{d \times d}$ and $\vec{1}_d \in \reals^d$
denote  the $d \times d$ identity matrix and the all-ones vector of size $d$,
respectively.

\begin{example}
  \label{example:uniform-runnning-upper-bound}
  Let $\Sigma = I_d$ and $\vmu = \frac{1}{\sqrt{d}} \vec{1}_d$. In the absence
  of the adversary, the optimal  Bayes classification error is $\bar{\Phi}(1)$.
  Moreover, simplifying the bounds in Corollary~\ref{cor:upper-bound-diagonal},
  we get
  \begin{equation*}
    \loss_{\vmu, \Sigma}(\mC_F^{(k)}, k) \leq \frac{1}{\sqrt{2 \log d}} + \bar{\Phi} \left( \sqrt{\frac{|F|}{d}} - \frac{16k\sqrt{2\log d}}{\sqrt{|F|}}\right).
  \end{equation*}
  This is minimized   when
  $F = [d]$, resulting
  in
  \begin{equation*}
    \optloss_{\vmu, \Sigma}(k) \leq \frac{1}{\sqrt{2 \log d}} + \bar{\Phi}\left( 1 - \frac{16k \sqrt{2\log d}}{\sqrt{d}} \right).
  \end{equation*}
  Note that if
  $k = o(\sqrt{d / \log d})$,
  the upper bound is approximately
  $\bar{\Phi}(1)$ which is the optimal classification error in the absence of
  the adversary. This means that for
  $k = o(\sqrt{d / \log d})$,
  the effect of the
  adversary can be completely neutralized. We will show a  lower bound for this
  example 
  later in Section~\ref{sec:lower-bound-ind} (see
  Example~\ref{example:uniform-running-lower-bound} therein) which shows that
  when
  $k \geq \sqrt{d} \log d$,
  no classifier  can do asymptotically better than a naive classifier. This
  establishes a phase transition at $ k = \sqrt{d}$ up to logarithmic terms.

\end{example}

\begin{example}
  \label{example:almost-uniform-running-upper-bound}
  Let $\Sigma = I_d$ and $\vmu = (d^{-\frac{1}{3}}, cd^{-\frac{1}{2}},
  cd^{-\frac{1}{2}}, \dots, cd^{-\frac{1}{2}})$ where $c$ is chosen such that
  $\snorm{\vmu}_2 = 1$, resulting in an optimal standard  error of
  $\bar{\Phi}(1)$ in the absence of the adversary. It turns out that the
  set $F$ that optimizes the bound in
  Corollary~\ref{cor:upper-bound-diagonal} is the set $[2: d]$, i.e.\ we need to
  {\color{black}discard} the first coordinate. In addition to this, we can see that if the
  classifier  does not
  {\color{black}discard} the first coordinate, it can neutralize adversarial \textcolor{black}{attacks with} budget of at most
  $d^{\frac{1}{3} - \epsilon}$, while  {\color{black}discarding} the first coordinate makes the
  classifier immune to adversarial budgets up to $d^{\frac{1}{2} - \epsilon}$.
  In fact, although the first coordinate is more informative compared to the
  other coordinates, due to this very same reason it is more susceptible to
  adversarial attacks, and it can do more harm than good when the input is
  adversarially corrupted. 
  This example highlights the importance of the filtration phase.
\end{example}

\subsection{Lower Bound on Optimal Robust Classification Error: Strategies for the Adversary}
\label{sec:lower-bound-ind}


In this section, we provide a lower bound on the optimal robust classification
error. This is accomplished by introducing an attack strategy for the adversary,
and showing that given such a fixed attack, no classifier can achieve better
than the lower bound that we introduce. 
The strategy is best understood when the covariance matrix
is diagonal. Therefore, we first assume that $\Sigma$ is diagonal and  denote the diagonal elements of $\Sigma$ by
$\sigma_1^2, \dots, \sigma_d^2$.  
We later use
 our strategy for diagonal covariance matrices
to get a general
lower bound for arbitrary $\Sigma$ (see Theorem~\ref{thm:general-lower-bound} at the
end of this section).

Assume that the  adversary observes realizations $(\vx, y) \sim \mD$ generated from the Gaussian
mixture model with parameters $\vmu, \Sigma$, where $\Sigma$ is diagonal.
A randomized strategy for the adversary with budget $k$ is identified by a probability 
distribution which upon observing such realizations $(\vx,y)$, generates a
random vector $\vxp$  that satisfies  $\pr{\snorm{\vxp - \vx}_0 \leq k \mid \vx,
  y} =1$.
The goal of the adversary is to design this randomized strategy  in a
way that the corrupted vector $\vxp$ bears very little information (or even no
information) about the label $y$. In this way, the loss in
\eqref{eq:general-opt-loss-def} will be  maximized.
Before rigorously defining our proposed strategy for the adversary, we
illustrated its main idea when $d=1$ in Figure~\ref{fig:converse}.

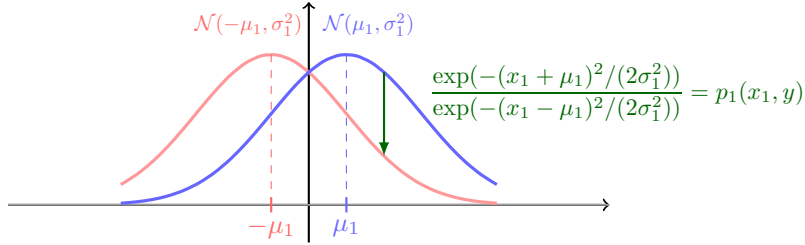
\begin{figure}
  \centering
    \begin{tikzpicture}
    \draw[thick, ->] (-4,0) -- (4,0);
    \draw[thick,->] (0,-0.5) -- (0,2.7);
    \draw[white] (-6.6,0) rectangle (6.6,3);
      
        \begin{scope}[xshift=-0.5cm]
          \draw[smooth, domain=-2:3, very thick, red!40] plot (\x,{2*exp(-\x*\x/2)});
      \node[red!60] at (0,-0.3) {$-\mu_1$};
      \draw[red!60, thick] (0,-0.1) -- (0,0.1);
      \draw[red!60, dashed] (0,0) -- (0,2);
      \node[red!60, scale=0.8] at (-0.3,2.4) {$\mathcal{N}(-\mu_1, \sigma_1^2)$};
    \end{scope}

    \begin{scope}[xshift=0.5cm]
      \draw[smooth, domain=-3:2, very thick,blue!60] plot (\x,{2*exp(-\x*\x/2)});
      \node[blue!60] at (0,-0.3) {$\mu_1$};
      \draw[blue!60, thick] (0,-0.1) -- (0,0.1);
      \draw[blue!60, dashed] (0,0) -- (0,2);
      \node[blue!60, scale=0.8] at (0.3,2.4) {$\mathcal{N}(\mu_1, \sigma_1^2)$};
      \draw[thick, -latex, green!40!black] (0.5,{2*exp(-0.5*0.5/2)}) -- (0.5,{2*exp(-1.5*1.5/2)});
    \end{scope}

    \node[anchor=west, green!40!black,scale=0.9] at (1.5,1.5) {$\displaystyle\frac{\exp(-(x_1+\mu_1)^2/(2\sigma_1^2))}{\exp(-(x_1-\mu_1)^2/(2\sigma_1^2))} = p_1(x_1, y)$};
  \end{tikzpicture}
  
  \caption{The idea behind our proposed strategy for the adversary when $d=1$.
    Assume $\mu_1 > 0$ and  the adversary observes a realization $(x_1, y)$ such that $y =
    1$, meaning that $x_1$ is a realization of $ \mN(\mu_1, \sigma_1^2)$ (i.e.\ the blue curve).
    If $x_1 \leq 0$, the adversary leaves it unchanged, i.e.\ $\xp_1 = x_1$. On the other hand, if $x_1
    > 0$,  we compute the ratio between the two densities (which is precisely
    $p_1(x_1, y)$ shown in the figure), and with probability $p_1(x_1, y)$ we
    pick $\xp_1$ from an arbitrary distribution (e.g. $\text{Uniform}[-1,1]$).
    When $y = -1$, we follow  a similar procedure, but reversed.
    It is easy to see that by doing so, the distribution of $\xp_1$ is the same
    when $y = 1$ and $y = -1$, hence $\xp_1$ bears no information about $y$.}
    \vspace{-.5cm}
  \label{fig:converse}
\end{figure}

Recall that $\vnu = \Sigma^{-1/2} \vmu$. Since $\Sigma$ is diagonal, $\nu_i =
\mu_i / \sigma_i$. We will fix a set of
coordinates $A \subseteq [d]$ and a specific value for the budget $k(A) =
\snorm{\vnu_A}_1 \log d$. We introduce a randomized strategy for the adversary  with the following
properties:  (i) it can change up to $k(A)$ coordinates of the input; and (ii)
all the changed coordinates belong to $A$, i.e. the coordinates in $A^c$ are
left untouched. We denote this adversarial strategy by $\adv(A)$.    Given $A \subset [d]$,
having observed $(\vx,y)$, $\adv(A)$ follows the procedure explained below. Let
$\vZ = (Z_1,\cdots, Z_d) \in \reals^d$ be a random vector that $\adv(A)$
constructs using the true input $\vx$. First of all, recall that $\adv(A)$ does
not touch the coordinates that are not in $A$, i.e. for $i \in A^c$ we let $Z_i
= x_i$.  For each $i \in A$, the adversary's act is simple: it either leaves the
value unchanged, i.e. $Z_i = x_i$, or it erases the value, i.e. $Z_i \sim
\text{Unif}[-1,1]$--a completely random value between $-1$ and $+1$. This binary
decision is encoded through a Bernoulli random variable $I_i$ taking value $0$
with probability $p_i(x_i,y)$ and value $1$ otherwise. Here $p_i(x_i,y)$ is
defined  as 
\vspace{-.2cm}
\begin{equation*}
  p_i(x_i, y) :=
  \begin{cases}
    \frac{\exp(-(x_i + y\mu_i)^2 / 2\sigma_i^2)}{\exp(-(x_i - y\mu_i)^2 / 2\sigma_i^2)}  & \text{if  }  \sgn(x_i) = \sgn(y\mu_i) \\
    0 & \text{otherwise}
  \end{cases}
\end{equation*}
Note that the condition $\sgn(x_i) = \sgn(y\mu_i)$ ensures that $p_i(x_i, y) \leq 1$. In summary, for each $i \in A$, $\adv(A)$ lets 
\begin{equation}
  \label{eq:adv-Zi-Ii-def}
  Z_i = x_i \times I_i + \text{Unif}[-1,1] \times (1-I_i),
\end{equation}
where $I_i = \text{Bernoulli} \left(1 - p_i(x_i, y_i) \right)$, and the random variables $I_i$ are generated completely independently w.r.t. all the other variables.  
It is easy to see that the following holds for the conditional density of
$\vZ_A$ given $y$ 
\vspace{-.05cm}
\begin{equation}
  \label{eq:adv-Z-cond-f}
  f_{\vZ_A|y}(\vec{z}_A | 1) = f_{\vZ_A|y}(\vec{z}_A | -1) = \prod_{i \in A} \left[  \frac{1}{\sqrt{2 \pi \sigma_i^2}} \exp\left( - \frac{(|z_i| + |\mu_i|)^2}{2\sigma_i^2}\right) + \frac{\alpha_i}{2} \one{z_i \in [-1,1]}   \right],
\end{equation}
\vspace{-.15cm}
where for $i \in A$
\begin{equation*}
  \alpha_i := \pr{I_i = 1 | y = 1} = \pr{I_i = 1| y = -1} = \int_0^\infty [1-p_i(t,1)] f_{x_i|y}(t|1) dt.
\end{equation*}
In other words, $\alpha_i$ is the probability of changing coordinate $i$. 
Finally, $\adv(A)$ checks if the vectors $\vZ$ and $\vx$ differ within the
budget constraint $k(A) :=  \snorm{\vnu_A}_1 \log d$. 
Define $\vxp$ as follows:
\begin{equation}
  \label{eq:vXp-def-A}
  \vxp :=
  \begin{cases}
    \vZ & \text{if } \sum_{i \in A} I_i \leq  \snorm{\vnu_A}_1 \log d\\
    \vx & \text{o.t.w.}
  \end{cases}
\end{equation}
It can be shown that with
high probability, $\vec{Z}$ is indeed within the specified budget and  $\vxp
= \vZ$.
From this definition, it is evident that with probability one we have
\begin{equation}
  \label{eq:vXp-vX-norm-0-log-d-nu-A-1}
  \snorm{\vxp - \vx}_0 \leq \snorm{\vnu_A}_1 \log d,
\end{equation}
and hence $\adv(A)$ is a randomized adversarial strategy that only changes the coordinates in $A$ and has budget $k(A) = \snorm{\vnu_A}_1 \log d$. 
Now we use this adversarial strategy to  show the following result. The proof of
Theorem~\ref{thm:lower-bound-diag} is given in Appendix~\ref{sec:app-lower-diagonal}.
\begin{thm}
  \label{thm:lower-bound-diag}
  Assume that  the covariance matrix $\Sigma$ is diagonal and let $\vnu = \Sigma^{-1/2}
  \vmu$. 
  Then for any subset $A \subseteq [d]$, we have
  \begin{equation*}
    \optloss_{\vmu, \Sigma}\bigl(\snorm{\vnu_A}_1 \log d \bigr) \geq \bar{\Phi}(\snorm{\vnu_{A^c}}_2) - \frac{1}{\log d}.
  \end{equation*}
\end{thm}

The main idea behind this result and the above adversarial strategy is that due
to~\eqref{eq:adv-Z-cond-f}, $\vZ_A$ is independent from $y$ and since the
coordinates of the input are independent from each other, and since with high probability $\vxp = \vZ$, the coordinates in $A$ have no
useful information for the classifier. Hence, the classifier can do no better
than the optimal Bayes classifier for the remaining  coordinates in $A^c$, which
results in 
 a classification error of $\bar{\Phi}(\snorm{\vnu_{A^c}}_2)$.

We now apply the bound of Thm~\ref{thm:lower-bound-diag} to
Examples~\ref{example:uniform-runnning-upper-bound}, \ref{example:almost-uniform-running-upper-bound} that we discussed 
in Section~\ref{sec:uppper-bound-result}.

\begin{example}
  \label{example:uniform-running-lower-bound}
  Assume that $\vmu$ and $\Sigma$ are as in
  Example~\ref{example:uniform-runnning-upper-bound}. Applying the bound in
  Theorem~\ref{thm:lower-bound-diag}, we get
  \begin{equation*}
    \optloss_{\vmu, \Sigma}\left( \frac{|A|}{\sqrt{d}} \log d \right) \geq \bar{\Phi}\left( \sqrt{1 - \frac{|A|}{d}} \right) - \frac{1}{\log d}.
  \end{equation*}
  Therefore, setting $A = [d]$, we obtain a lower bound of almost $\bar{\Phi}(0)
  = 1/2$ for adversarial budget $\sqrt{d} \log d$. In other words, if the
  adversarial budget is more than $\sqrt{d} \log d$, asymptotically no
  classifier can do better than a random guess. This together with the
  discussion in Example~\ref{example:uniform-runnning-upper-bound}
  establishes a phase transition around $\sqrt{d}$ (modulo logarithmic terms).
\end{example}

\begin{example}
\label{example:almost-uniform-running-lower-bound}
  Assume that $\vmu$ and $\Sigma$ are as in
  Example~\ref{example:almost-uniform-running-upper-bound}. Applying the bound
  of Theorem~\ref{thm:lower-bound-diag} with $A = [d]$, we obtain 
  $\optloss_{\vmu, \Sigma}(k) \geq \bar{\Phi}(0) - 1/\log d \approx 1/2$ where
  $k = (d^{-\frac{1}{3}} + c (d-1) / \sqrt{d}) \log d \approx \sqrt{d} \log d$.
  Hence, comparing this to
  Example~\ref{example:almost-uniform-running-upper-bound}, we find 
  similar to Example~\ref{example:uniform-running-lower-bound} above that a phase
  transition occurs around adversarial budget $\sqrt{d}$ up to logarithmic
  terms. 
 \end{example}

Now we state our general lower bound which holds for an arbitrary covariance
matrix. 
This is  Theorem~\ref{thm:general-lower-bound} below, whose proof is  provided in Appendix~\ref{sec:app-general-lower}.
Given $\vmu$ and $\Sigma$, we define the $d \times d$ matrix $R$ where the
$i,j$ entry in $R$ is $R_{i,j} = \Sigma_{i,j} / \sqrt{\Sigma_{i,i} \Sigma_{j,j}}$.
In other words, $R_{i,j}$ is the correlation coefficient between the $i$th and
the $j$th coordinates in our Gaussian noise. Equivalently, with $\tSigma$ being
the diagonal part of $\Sigma$, we may write
\begin{equation}
  \label{eq:R-def}
  R := \tSigma^{-\frac{1}{2}} \Sigma \tSigma^{-\frac{1}{2}}.
\end{equation}
It is evident that since $\Sigma$ is assumed to be positive definite, $R$ is
also positive definite. 
Furthermore, we
define $\vu = (u_1, \dots, u_d)$ where
\begin{equation}
  \label{eq:vu-def}
    u_i = \frac{\mu_i}{\sqrt{\Sigma_{i,i}}} \qquad 1 \leq i \leq d.
  \end{equation}
  
\begin{thm}
  \label{thm:general-lower-bound}
  With $\vu$ and $R$ defined as in~\eqref{eq:R-def} and \eqref{eq:vu-def}
  respectively, for all $A \subseteq [d]$, we have
  \begin{equation*}
    \optloss_{\vmu, \Sigma}\left( \frac{1}{\sqrt{\zeta_\text{min}}} \snorm{\vu_A}_1 \log d \right) \geq \bar{\Phi}(\snorm{\vu_{A^c}}_2) - \frac{1}{\log d},
  \end{equation*}
  where $\zeta_\text{min} > 0$ denotes the minimum eigenvalue of $R$. 
\end{thm}

\begin{rem}
  Note that when $\Sigma$ is diagonal, we have $R = I_d$, $\zeta_\text{min} =
  1$, and $\vu = \vnu =
  \Sigma^{-1/2} \vmu$.  Therefore, the bound in
  Theorem~\ref{thm:general-lower-bound} reduces to that of Theorem~\ref{thm:lower-bound-diag}.
\end{rem}

\subsection{Optimality of $\alg$ in the diagonal regime}
\label{sec:diagonal-compare-asymp}

We have already seen for  our two running examples that up to logarithmic
terms, our lower and upper bounds match 
(Examples~\ref{example:uniform-runnning-upper-bound} and
\ref{example:almost-uniform-running-upper-bound} for upper bound, and their
matching lower bounds in Examples~\ref{example:uniform-running-lower-bound}
and~\ref{example:almost-uniform-running-lower-bound}, respectively). First, in
Section~\ref{sec:diagonal-comparing-the-bounds}, we show that our lower and upper bounds indeed match up to logarithmic
terms in the \emph{diagonal regime}, i.e.\ when the
covariance matrix is diagonal.
Then, in Section
This in particular implies that our robust
classification algorithm $\alg$ is optimal in this  regime. 



\subsubsection{Comparing the Bounds}
\label{sec:diagonal-comparing-the-bounds}

In Theorem~\ref{thm:ind-bound-finite} below, in the diagonal regime we compare
our upper bound of Corollary~\ref{cor:upper-bound-diagonal} and our lower bound
of Theorem~\ref{thm:lower-bound-diag}. Proof of
Theorem~\ref{thm:ind-bound-finite} is given in Appendix~\ref{sec:app_diagonal_match}.
Recall that $\vnu := \Sigma^{-1/2} \vmu$ and we assume~\eqref{eq:norm-nu-is-one}
holds. 
When $\Sigma$ is diagonal and its
diagonal entries are $\sigma_1^2, \dots, \sigma_d^2$, we have $\nu_i = \mu_i / \sigma_i$.
Without loss of generality, we may assume that the coordinates of
$\vnu$ are
decreasingly ordered such that 
\begin{equation}
 \label{eq:vnu-ordered-assumption}
 |\nu_1| \geq |\nu_2| \geq \dots \geq |\nu_d|. 
\end{equation}
Given $c \in [0,1]$, we define
\begin{equation}
  \label{eq:lambda-c-def}
  \lambda_c := \min \{\lambda: \snorm{\vnu_{[1:\lambda]}}_2 \geq c
  \}.
\end{equation}


\begin{thm}
  \label{thm:ind-bound-finite}
  If $\Sigma$ is diagonal and the coordinates in $\vnu$ are sorted as
  in~\eqref{eq:vnu-ordered-assumption},  then:
  \begin{enumerate}
  \item For $0 \leq c < 1$, we have
  \begin{equation*}
    \optloss_{\vmu, \Sigma}\left( \frac{\snorm{\vnu_{[1:\lambda_c]}}_1}{\log d} \right) \leq \frac{1}{\sqrt{2 \log d}} + \bar{\Phi}\left(
      \sqrt{1-c^2} - \frac{16\sqrt{2}}{\sqrt{1-c^2} \sqrt{\log d}} \right ).
  \end{equation*}
  \item For $0 < c \leq 1$, we have
  \begin{equation*}
    \optloss_{\vmu, \Sigma}(\snorm{\vnu_{[1:\lambda_c]}}_1 \log d) \geq \bar{\Phi}(
    \sqrt{1-c^2}) - \frac{1}{\log d}.    
\end{equation*}
\end{enumerate}
\end{thm}

\begin{rem}
  \label{rem:matching-bounds-discussion}
Roughly speaking, Theorem~\ref{thm:ind-bound-finite} says that up to logarithmic
terms, we have
\begin{equation*}
  \optloss_{\vmu, \Sigma}(\snorm{\vnu_{[1:\lambda_c]}}_1) \approx \bar{\Phi}(
  \sqrt{1-c^2}).
\end{equation*}
Recall from our previous discussion that we are interested in studying
adversarial budgets scaling as $d^\alpha$, which justifies neglecting the
multiplicative logarithmic terms. 
Furthermore, following the proof of Theorem~\ref{thm:ind-bound-finite}, the
upper bound in the first part is obtained by our robust classifier
by setting 
$F = \{\lambda_c, \dots, d\}$. Roughly speaking, the classifier discards
the  coordinates in $\vnu$ which constitute fraction $c$ of the  $\ell_2$ norm of $\vnu$, and performs
a truncated inner product classification on the remaining coordinates.
But the
$\ell_2$ norm of the remaining coordinates is roughly $
\sqrt{1-c^2}$, and the effect of
truncation is vanishing as long as the adversarial power is below
$\snorm{\vnu_{[1:\lambda_c]}}_1$ by a logarithmic factor. Note that although the
top coordinates in $\vnu$ are relatively more important in terms of the
classification power, due to the same reason,
they are more susceptible  to  adversarial attack.
\end{rem}

\begin{rem}
  \label{rem:diagonal-how-to-chose-F}
  In view of Theorem~\ref{thm:ind-bound-finite} and
  Remark~\ref{rem:matching-bounds-discussion}, we can introduce the following
  mechanism for choosing the surviving set  $F$ for the adversary given
  adversarial power $k$. Let $r(k) = \min\{r: \snorm{\vnu_{[1:r]}}_1 \geq k \log
  d\}$ and set $F = [r(k):d]$. Then the classifier $\mC^{(k)}_F$ achieves the
  optimal robust classification error of almost
  $\bar{\Phi}(\sqrt{1-c^2})$ where $c = \snorm{\vnu_{[1:r(k)]}}_2$.
\end{rem}


\subsubsection{Asymptotic Analysis, Phase Transitions, and Trade-offs}
\label{sec:asymptotic-diagonal}

\ifcolt
\else
 \begin{figure}
  \centering
    \scalebox{0.9}{
\begin{tikzpicture}

      \begin{scope}[xshift=-3.3cm]
    \draw[thick, ->] (0,0) -- (4,0);
    \draw[thick, ->] (0,0) -- (0,2);

    \draw[dashed, green!40!black, thick] (0,0.475) -- (3,0.475);
    
    \draw[very thick, blue!60] (0,0.475) -- (1.5,0.475) -- (1.5,1.5) -- (3,1.5);
    \coordinate (a) at (1.5,1.5);
    \coordinate (b) at (1.5,0);
    \coordinate (c) at (3,0.475);
    \coordinate (d) at (1.5,0.475);
    
    \draw[very thick, blue!60] (a) -- (3,1.5);

    \draw[dashed, blue!60, thick] (3,0) -- (3,1.5)
    (0,1.5) -- (a)
    (b) -- (d);

    \node[scale=0.7] at (3,-0.3) {$1$};
    \node[scale=0.7] at (0,-0.3) {$0$};
    \node[scale=0.7] at (1.5,-0.3) {$\alpha_0$};
    
    \node[scale=0.9] at (1.5,-0.7) {(a)};
    
    \node[anchor=east, scale=0.7] at (-0.05,1.5) {$\frac{1}{2}$};
    \node[anchor=east,scale=0.7] at (-0.05,0.475) {$\bar{\Phi}(1)$};

    \node[scale=0.85] at (4.5,-0.3) {$\log_d k$};
    \node[scale=0.85] at (0,2.3) {error};

    \node[scale=0.7, text width=2.4cm, align=center, green!40!black] (osenode) at (4.5,.75) {Optimal \\Standard Error};
    \coordinate (osebegin) at (2.5, 0.475);
    \path[thick, ->, green!40!black, bend left] (osebegin) edge (osenode.west);

    \node[scale=0.7, text width=3.1cm, align=center, blue!70] (robnode) at (4.5,1.9) {Optimal Robust \\Classification  Error};
    \path[thick, ->, blue!60, bend left] (2.6,1.5) edge (robnode.west);
  \end{scope}

  \begin{scope}[xshift=3.3cm]
    \draw[thick, ->] (0,0) -- (4,0);
    \draw[thick, ->] (0,0) -- (0,2);

    \draw[dashed, green!40!black, thick] (0,0.475) -- (3,0.475);
    
    \coordinate (a) at (1.5,1.5);
    \coordinate (b) at (1.5,0);
    \coordinate (c) at (3,0.475);
    \coordinate (d) at (1.5,0.475);

    \draw[dashed, blue!60, thick] (3,0) -- (3,1.5)
    (0,1.5) -- (3,1.5);


    \draw[smooth, very thick, blue!60, domain=0:3] plot (\x, {0.475 +(1.5-0.475)/(1+exp(-5*(\x-1.5))});
    
    \node[scale=0.7] at (3,-0.3) {$1$};
    \node[scale=0.7] at (0,-0.3) {$0$};
    \node[scale=0.9] at (1.5,-0.7) {(b)};
    \node[anchor=east, scale=0.7] at (-0.05,1.5) {$\frac{1}{2}$};
    \node[anchor=east,scale=0.7] at (-0.05,0.475) {$\bar{\Phi}(1)$};

    \node[scale=0.85] at (4.5,-0.3) {$\log_d k$};
    \node[scale=0.85] at (0,2.3) {error};
    \node[scale=0.7, text width=2.4cm, align=center, green!40!black] (osenode) at (4.5,.75) {Optimal \\Standard Error};
    \coordinate (osebegin) at (2.5, 0.475);
    \path[thick, ->, green!40!black, bend left] (osebegin) edge (osenode.west);

        \node[scale=0.7, text width=3.1cm, align=center, blue!70] (robnode) at (4.5,1.9) {Optimal Robust \\Classification  Error};
    \path[thick, ->, blue!60, bend left] (2.6,1.5) edge (robnode.west);

  \end{scope}
\end{tikzpicture}
}

  \caption{Asymptotic behavior in the diagonal regime: Illustration of scenarios with (a) a phase transition, and (b) no phase transition\label{fig:phase-transition-intuition}}
   \vspace{-3mm}
 \end{figure}
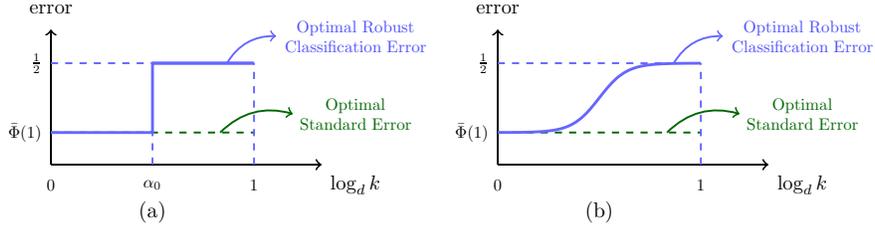

 \fi

In this section, we perform a thorough
 analysis when the adversarial budget scales as $d^\alpha$ using  
our  results in the diagonal regime. Here, we describe the main messages. (i) We show  that our bounds asymptotically match in the
 diagonal regime and $\alg$ is indeed optimal. (ii) Through the asymptotic analysis,
 we observe that in some scenarios, a sharp phase transition on the optimal robust error occurs as we increase $\alpha := \log_d k$ (See Figure~\ref{fig:phase-transition-intuition}-(a)). We have already given examples of such scenarios (e.g.\ Example~\ref{example:uniform-runnning-upper-bound}). In such cases, below the transition, i.e. when $\alpha < \alpha_0$, the optimal robust error is the same as the optimal standard error. And when we are above the transition, i.e. when $\alpha > \alpha_0$, any classifier becomes useless as the robust error becomes $\frac 12$. 
 As a result, asymptotically speaking, there exists no tradeoff between robustness and standard accuracy in scenarios where there is a sharp transition.

 However, 
 there are other scenarios where instead of a sharp phase transition, in the
 asymptotic regime, the optimal
 robust 
 error continuously increases as a function of adversary's
 budget (see Figure~\ref{fig:phase-transition-intuition}-(b)). In such scenarios, there exists a non-trivial tradeoff between robustness and standard accuracy. I.e. to achieve optimal robust error it is necessary to filter 
 many 
 informative coordinates which hurts the standard accuracy. 
 See Example~\ref{example:log-block} below.

In order to perform an asymptotic analysis, we assume that 
the dimension of the space, $d$,
  goes to infinity. More precisely, we assume that we have a sequence
$(\vmud, \Sigmad)$ where for each $d$, $\vmud \in \reals^{d}$ and
$\Sigmad$ is a diagonal covariance matrix with nonzero diagonal entries. We
define
\begin{equation*}
  \vnud := (\Sigmad)^{-1/2} \vmud.
\end{equation*}
As usual, as in \eqref{eq:norm-nu-is-one}, in order to keep the optimal classification error in the absence of
the adversary fixed, we assume that 
\begin{equation}
\label{eq:vnud-norm-2-1}
  \snorm{\vnud}_2 = 1 \qquad \forall d.
\end{equation}
Furthermore, without loss of generality, we assume that the coordinates in
$\vnu$ are sorted in a descending order with respect to their magnitude, i.e.\ 
\begin{equation}
\label{eq:vnud-sorted}
  |\nud_1| \geq |\nud_2| \geq \dots \geq |\nud_d| \qquad \forall d.
\end{equation}
To simplify the notation, we use $\optloss_d(.)$ as a shorthand for
$\optloss_{\vmud, \Sigmad}(.)$. We are mainly interested in studying the
asymptotic behavior of $\optloss_d(k_d)$ when $k_d$ is a sequence of
adversarial budgets so that $k_d$ behaves like $d^\alpha$. Motivated by
Theorem~\ref{thm:ind-bound-finite}, it is natural to define
\begin{equation}
  \label{eq:lambda-d-c-def}
  \lambdad_c := \min\{\lambda: \snorm{\vnud_{[1:\lambda]}}_2 \geq c\} \qquad \text{for } 0 < c \leq 1.
\end{equation}
Furthermore, for $0 < c \leq 1$, we define
\begin{equation}
  \label{eq:asymp-func-Psi-d-def}
  \funcd(c) := \log_d \snorm{\vnud_{[1:\lambdad_c]}}_1.
\end{equation}
Note that since $c  > 0$, $\lambdad_c \geq 1$ and
$\snorm{\vnud_{[1:\lambdad_c]}}_1 > 0$. Therefore, $\funcd(c)$  is well-defined. 
Furthermore, it is easy to verify the following properties for the function $\funcd(.)$:
\begin{lem}
  \label{lem:funcd-props}
  $\funcd(.)$ is nonincreasing and $\funcd(c) \in [-1/2,1/2]$ for all $c \in
  (0,1]$. 
\end{lem}

\begin{proof}
  Note that
  \begin{equation*}
    \funcd(c) = \log_d \snorm{\vnud_{[1:\lambdad_c]}}_1 \leq \log_d \snorm{\vnud}_1 \leq \log_d (\sqrt{d} \snorm{\vnud}_2) = \log_d \sqrt{d} = \frac{1}{2}.
  \end{equation*}
  On the other hand, note that for $c > 0$, we have $\lambdad_c \geq 1$ and
  $\funcd(c) \geq \log_d |\nud_1| = \log_d \snorm{\vnu}_\infty$. Furthermore, we
  have $    1 = \snorm{\vnud}_2^2 \leq d \snorm{\vnud}_\infty$ which implies
  that $\snorm{\vnud}_\infty \geq 1/ \sqrt{d}$. Consequently, $\funcd(c) \geq
  \log_d 1/\sqrt{d} = -1/2$. This completes the proof.
\end{proof}

Roughly speaking, Theorem~\ref{thm:ind-bound-finite} implies that if $k_d$
behaves like $d^{\funcd(c)}$, then $\optloss(k_d) \approx
\bar{\Phi}(\sqrt{1-c^2})$. In order to transform this into a formal asymptotic
argument, we assume that for all $c \in (0,1]$, the sequence $\funcd(c)$ is
convergent, and we define $\funcinf(c) := \lim_{d \rightarrow \infty} \funcd(c)$
as the limit. Since $\funcd(.)$ is nondecreasing, if the pointwise limit
$\funcinf(.)$ exists, it is also nondecreasing and we may define
\begin{equation*}
  \funcinf(0) := \lim_{c \downarrow 0} \funcd(c).
\end{equation*}
Additionally, we can show the following lemma.
\begin{lem}
  \label{lem:func-inf-nonnegative}
  If $\funcinf(.)$  exists as above, then $\funcinf(c) \in [0,1/2]$
  for all $c \in [0,1]$. 
\end{lem}

\begin{proof}
  For all $c>0$ and all $d$, we have
  \begin{equation*}
    \snorm{\vnud_{[1:\lambdad_c]}}_1 \geq \snorm{\vnud_{[1:\lambdad_c]}}_2^2 \geq c^2.
  \end{equation*}
  Therefore
  \begin{equation*}
    \funcinf(c) = \lim_{d \rightarrow \infty} \funcd(c) = \lim_{d \rightarrow \infty} \log_d \snorm{\vnud_{[1:\lambdad_c]}}_1 \geq \liminf_{d \rightarrow \infty} 2 \log_d c = 0.
  \end{equation*}
  Sending $c$ to zero we also realize that $\funcinf(0) \geq 0$.
\end{proof}

Given these, we can formalize the following asymptotic behavior for the optimal
robust classification error. The proof of Theorem~\ref{thm:asymp-independent}
below is given in Appendix~\ref{app:asymp-theorem-proof}.
\begin{thm}
  \label{thm:asymp-independent}
  If $\funcd(.)$ converges pointwise to a nondecreasing function $\funcinf:[0,1]
  \rightarrow [0,1/2]$ as above, then
  the following hold for all $c \in [0,1]$:
  \begin{enumerate}
  \item If $\limsup_{d \rightarrow \infty} \log_d k_d < \funcinf(c)$, then
    $\limsup_{d \rightarrow \infty} \optloss_d(k_d) \leq
    \bar{\Phi}(\sqrt{1-c^2})$.
  \item If If $\liminf_{d \rightarrow \infty} \log_d k_d > \funcinf(c)$, then
    $\liminf_{d \rightarrow \infty} \optloss_d(k_d) \geq 
    \bar{\Phi}(\sqrt{1-c^2})$.
  \end{enumerate}
\end{thm}

It is sometimes more convenient to state the above theorem in terms of the
pseudo inverse of the function $\funcinf(.)$ defined as follows. 
For $\alpha \in [0,1]$, we define
\begin{equation}
  \label{eq:funcinf-inv-def}
  \funcinfinv(\alpha) := \inf \{ \bar{\Phi}(\sqrt{1-c^2}): \funcinf(c) \geq \alpha\} \wedge \frac{1}{2}.
\end{equation}
Note that since $\funcinf(c) \leq 1/2$ for all $c \in [0,1]$, we have
\begin{equation*}
  \funcinfinv(\alpha) = \frac{1}{2} \qquad \forall c > \frac{1}{2}.
\end{equation*}
With this, we can restate Theorem~\ref{thm:asymp-independent} as follows.
\begin{cor}
  \label{cor:ind-asymp-inv}
  In the setup of Theorem~\ref{thm:asymp-independent}, for $\alpha \in [0,1]$ we
  have
  \begin{enumerate}
  \item If $\limsup \log_d k_d < \alpha$ then $\limsup \optloss_d(k_d) \leq
    \funcinfinv(\alpha)$.
  \item If $\liminf \log_d k_d > \alpha$ then $\liminf \optloss_d(k_d) \geq \funcinfinv(\alpha)$.
  \end{enumerate}
\end{cor}


We now discuss this asymptotic result through some examples.

\begin{example}
  \label{example:uniform-asymptotic}
   Let $\vmud$ and $\Sigmad$ be as in
   Example~\ref{example:uniform-runnning-upper-bound}, i.e.\ $\Sigmad =
   I_d$ and $\vmud = \frac{1}{\sqrt{d}} \vec{1}_d$. Therefore, we have
  \begin{equation*}
    \vnud = (\Sigmad)^{-\frac{1}{2}} \vmud = \left( \frac{1}{\sqrt{d}}, \frac{1}{\sqrt{d}}, \dots, \frac{1}{\sqrt{d}} \right).
  \end{equation*}
  Using~\eqref{eq:lambda-d-c-def}, we have $\lambdad_c = \lfloor dc^2 \rfloor$
  and
  \begin{equation*}
    \funcd(c) = \log_d \snorm{\vnud_{[1:\lambdad_c]}}_1 = \log_d \frac{\lfloor dc^2 \rfloor}{\sqrt{d}} = \frac{1}{2} + o(1).
  \end{equation*}
  Therefore, sending $d \rightarrow \infty$, we realize that
  \begin{equation*}
    \funcinf(c) = \frac{1}{2} \qquad \forall c \in [0,1].
  \end{equation*}
  Moreover, using~\eqref{eq:funcinf-inv-def}, we get
  \begin{equation*}
    \funcinfinv(\alpha) =
    \begin{cases}
      \bar{\Phi}(1) & \alpha \leq \frac{1}{2} \\
      \frac{1}{2} & \alpha > \frac{1}{2}.
    \end{cases}
  \end{equation*}
  Figure~\ref{fig:uniform-example-asymptotic} illustrates $\funcinf(.)$ and $\funcinfinv(.)$ for this example.
  Therefore, employing Corollary~\ref{cor:ind-asymp-inv}, we realize that
  \begin{enumerate}
  \item If $\limsup \log_d k_d < 1/2$ then $\limsup \optloss_d(k_d) \leq
    \bar{\Phi}(1)$
  \item If $\liminf \log_d k_d > 1/2$ then $\optloss(k_d) \geq 1/2$.
  \end{enumerate}
  In other words, we observe  a phase transition around $\sqrt{d}$ in the sense that
  if the adversary's budget is asymptoticallly below $\sqrt{d}$, the classifier
  can achieve the robust classification error $\bar{\Phi}(1)$, i.e.\ as if there
  is no adversary, while if the adversary's budget is asymptotically above
  $\sqrt{d}$,  no classifier can achieve a robust classification error better
  than that of a trivial classifier. This is consistent with the previous
  observations in this case, i.e.
  Examples~\ref{example:uniform-runnning-upper-bound} and \ref{example:uniform-running-lower-bound}.
\end{example}

\begin{figure}
  \centering
  \begin{tikzpicture}
  \begin{scope}[xshift=-4cm]
    \draw[thick, ->] (-1,0) -- (4,0);
    \draw[thick, ->] (0,-0.5) -- (0,2);
    \draw[very thick, blue!60] (0,1.5) -- (3,1.5);
    \draw[thick, dashed, blue!60] (3,0) -- (3,1.5);
    \node at (3,-0.3) {$1$};
    \node[anchor=east] at (-0.05,1.5) {$\frac{1}{2}$};

    \node at (4.2,-0.2) {$c$};
    \node at (-0.4,2.3) {$\Psi_\infty(c)$};
  \end{scope}

    \begin{scope}[xshift=4cm]
    \draw[thick, ->] (-1,0) -- (4,0);
    \draw[thick, ->] (0,-0.5) -- (0,2);

    \draw[very thick, blue!60] (0,0.475) -- (1.5,0.475);
    \node[fill=blue!60, circle, inner sep=1.5pt] at (1.5,0.475) {};
    \node[draw=blue!60, thick, circle, inner sep=1.5pt] (a) at (1.5,1.5) {};
    \draw[very thick, blue!60] (a) -- (3,1.5);

    \draw[dashed, blue!60, thick] (3,0) -- (3,1.5)
    (0,1.5) -- (a);
    
    \node at (3,-0.3) {$1$};
    \node[anchor=east] at (-0.05,1.5) {$\frac{1}{2}$};
    \node[anchor=east,scale=0.8] at (-0.05,0.475) {$\bar{\Phi}(1)$};

    \node at (4.2,-0.2) {$\alpha$};
    \node at (-0.4,2.3) {$\Psi^{-1}_\infty(\alpha)$};
  \end{scope}

\end{tikzpicture}

  \caption[Uniform Example Asymptotic]{$\funcinf(.)$ and $\funcinfinv(.)$ for
    Example~\ref{example:uniform-asymptotic}. This observe a phase transition at $\sqrt{d}$
  where below this threshold, adversary's effect can completely be neutralized,
  while above this threshold, the classifier can only achieve the trivial bound.}
  \label{fig:uniform-example-asymptotic}
\end{figure}
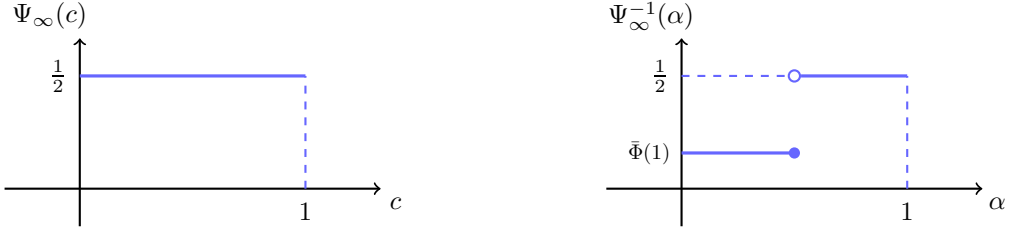

It is interesting to observe that not always we have a phase transition as in
the above example. Below we discuss an example in which we have no phase
transition, and the asymptotic robust classification error gradually increases
as a function of the adversary's budget.

\begin{example}
  \label{example:log-block}
  Let $\Sigma = I_d$. Assume that $d = 2^n - 1$ for some integer $n$ and define
  \begin{equation*}
    \vmud = \left( \frac{\sqrt{1/n}}{1}, \frac{\sqrt{1/n}}{\sqrt{2}}, \frac{\sqrt{1/n}}{\sqrt{2}}, \dots, \frac{\sqrt{1/n}}{\sqrt{d/2}}, \dots, \frac{\sqrt{1/ n}}{\sqrt{d/2}}\right).
  \end{equation*}
  More precisely, we split the unit $\ell_2$ norm of $\vmud$ into $n$
  blocks, where the first block is the first coordinate, the second block is the
  second two coordinate, the $i$th block constitutes of $2^i$ coordinates, and
  the final block is the last $d/2$ coordinates. Moreover, the power is
  uniformly distributed within each block. It is easy to see that for $c =
  \sqrt{m/n}$ for $1 \leq m \leq n$, we have $\lambdad_c = 2^m-1$ and
  \begin{equation*}
    \funcd(c) = \funcd\left( \sqrt{\frac{m}{n}} \right) = \log_d\left( \sqrt{\frac{1}{n}} \frac{\sqrt{2}^{m}- 1}{\sqrt{2}-1} \right) = \frac{c^2}{2} + o(1).
  \end{equation*}
  Therefore, $\funcd(.)$ converges pointwise to $\funcinf(.)$ such that
  $\funcinf(c) = c^2/2$ for $0 \leq c \leq 1$. Thereby, we have
  \begin{equation*}
    \funcinfinv(\alpha) =
    \begin{cases}
      \bar{\Phi}(1-2 \alpha) & 0 \leq \alpha \leq 1/2 \\
      \frac{1}{2} & 1/2 < \alpha \leq 1.
    \end{cases}
  \end{equation*}
  Figure~\ref{fig:log-block-example} illustrates $\funcinf(.)$ and $\funcinfinv(.)$ in this examples. As
  we can see, unlike Example~\ref{example:uniform-asymptotic}, we do not have a phase transition here. In fact,
  the asymptotic optimal robust classification error continuously increases as a
  function of adversarial $\ell_0$ budget. 
\end{example}

\begin{figure}
  \centering
  \begin{tikzpicture}
  \begin{scope}[xshift=-4cm]
    \draw[thick, ->] (-1,0) -- (4,0);
    \draw[thick, ->] (0,-0.5) -- (0,2);
    \draw[very thick, blue!60, domain=0:3] plot (\x,\x*\x*1.5/9);
    \draw[thick, dashed, blue!60] (3,0) -- (3,1.5)
    (0,1.5) -- (3,1.5);
    \node at (3,-0.3) {$1$};
    \node[anchor=east] at (-0.05,1.5) {$\frac{1}{2}$};

    \node at (4.2,-0.2) {$c$};
    \node at (-0.4,2.3) {$\Psi_\infty(c)$};
  \end{scope}

    \begin{scope}[xshift=4cm]
    \draw[thick, ->] (-1,0) -- (4,0);
    \draw[thick, ->] (0,-0.5) -- (0,2);

    \draw[very thick, blue!60]
    ( 0., 0.475966) -- 
( 0.1, 0.525972) -- 
( 0.2, 0.579187) -- 
( 0.3, 0.635566) -- 
( 0.4, 0.695033) -- 
( 0.5, 0.757478) -- 
( 0.6, 0.822759) -- 
( 0.7, 0.890704) -- 
( 0.8, 0.961108) -- 
( 0.9, 1.03373) -- 
( 1., 1.10832) -- 
( 1.1, 1.18459) -- 
( 1.2, 1.26222) -- 
( 1.3, 1.34089) -- 
( 1.4, 1.42027) -- 
( 1.5, 1.5) --
(3,1.5);

\draw[blue!60, dashed, thick] (3,0) -- (3,1.5) (0,1.5) -- (1.5,1.5);

    \node at (3,-0.3) {$1$};
    \node[anchor=east] at (-0.05,1.5) {$\frac{1}{2}$};
    \node[anchor=east,scale=0.8] at (-0.05,0.475) {$\bar{\Phi}(1)$};

    \node at (4.2,-0.2) {$\alpha$};
    \node at (-0.4,2.3) {$\Psi^{-1}_\infty(\alpha)$};
  \end{scope}

\end{tikzpicture}

  \caption{$\funcinf(.)$ and $\funcinfinv(.)$ for
    Examples~\ref{example:log-block}. Unlike Example~\ref{example:uniform-asymptotic}, we do not have a phase transition here and the asymptotic optimal robust classification error continuously increases as a
  function of the adversarial $\ell_0$ budget.}
  \label{fig:log-block-example}
\end{figure}
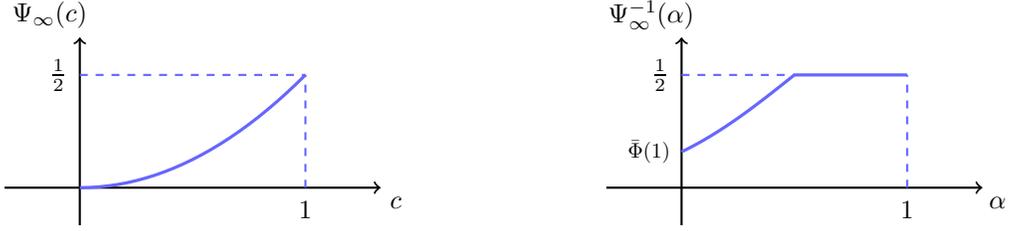

\vspace{-.2cm}
\section{Conclusion}
In this paper, we studied the binary Gaussian mixture model under $\ell_0$ attack. We developed a novel nonlinear classifier called $\alg$ that first cleverly selects the robust coordinates of the input and then classifies based on a truncated inner product operation. Analyzing the performance of our proposed method, we derived an upper bound on optimal robust classification error. We further derived a lower bound on this, and showed the efficacy of $\alg$: when the covariance matrix of Gaussian mixtures is diagonal, $\alg$ is asymptotically optimal.

There are many directions to be pursued. Deriving a tighter lower bound and resolving the optimality gap for the case of non-diagonal covariance matrices remains open. Applying the key ideas of $\alg$, filtration and truncation, to a more complicated setting (e.g. neural networks) can be of great importance from a practical viewpoint. A crucial message of this paper is to emphasize the importance of nonlinear operations such as truncation for designing defense against $\ell_0$ attacks. Finally, analyzing robust classification error with $\ell_0$ attacks for more complex stylized models such as multi-class Gaussian mixtures, two-layer neural networks, neural tangent kernel models, etc. is a promising future direction.

\newcommand{\etalchar}[1]{$^{#1}$}

\newpage

\appendix

\editstart

\section{Proof of Lemma~\ref{lem:trun-ip-bound}}
\label{sec:trunc-stat}

In this section, we prove Lemma~\ref{lem:trun-ip-bound}. First  we
need to define some notations and discuss  some lemmas.

Given  $\vx = (x_1, \dots, x_d) \in \reals^d$, we define the sample average of $\vx$ as $\mean(\vx) :=
\sum_{i=1}^d x_i/d$. Moreover, we define \emph{truncated sum} $\tsum_k(\vx)$ for $k < n/2$ as follows. Let $x_{(1)} \leq
x_{(2)} \leq \dots \leq x_{(n)}$ be the set of sorted values in $\vx$. We define
\begin{equation*}
  \tsum_k(\vx) := \sum_{i=k+1}^{d-k} x_{(i)},
\end{equation*}
which is the truncated sum of the elements in $\vx$ after removing the top and
bottom $k$ values.  For instance, $\tsum_1(1,1,2,3,4,5) = 1 + 2 + 3 + 4 = 10$. Moreover, we define the truncated mean of $\vx$ as follows:
\begin{equation*}
  \tmean_k(\vx) := \frac{\tsum_k(S)}{d-2k}.
\end{equation*}
Note that when $k=0$, the above quantities reduce to the sum and the sample
average, respectively. 
It is straightforward to see that
\begin{equation}
  \label{eq:tmean-mean-bounded-samples}
  \left|\tsum_k(\vx) - \sum_{i=1}^n x_i \right| \leq 2k M \qquad \qquad \text{ given } |x_i| \leq M \,\,\,\, \forall 1 \leq i \leq n.
\end{equation}

\begin{lem}
\label{lem:k-modify-survivals-are-bounded}
  Assume that $\vx = (x_1, \dots, x_d) \in \reals^d$ and $\vxp = (\xp_1, \dots,
  \xp_d) \in \reals^d$ are given such that $\vxp$ is identical to $\vx$ in all
  but at most $k < d /2$ coordinates, i.e.\ 
  $\snorm{\vx - \vxp}_0 \leq k$.
Moreover, assume that for some $M < \infty$, we have  $|x_i| \leq
  M$ for all $1 \leq i \leq d$.  Then, if $\xp_{(1)} \leq \xp_{(2)} \leq \dots \leq
  \xp_{(d)}$ are the sorted coordinates in $\vxp$, we have
  \begin{equation*}
    |\xp_{(i)}| \leq M \qquad \forall \, k+1 \leq i \leq d-k.
  \end{equation*}
\end{lem}

Essentially, what Lemma~\ref{lem:k-modify-survivals-are-bounded} states is
that if we modify at most $k$ coordinates in a vector whose elements are
bounded by $M$, in the resulting vector, after truncating the top and bottom $k$
coordinates, all the surviving values are also bounded by $M$.

\begin{proof}[Proof of Lemma~\ref{lem:k-modify-survivals-are-bounded}]
  Let $i_1, \dots, i_l$ for $l \leq k$ be the coordinates where $\vxp$ differs
  from $\vx$, i.e.\ $x_{i_j} \neq \xp_{i_j}$ for $1 \leq j \leq l$. Note that if
  $|\xp_{i_j}| > M$ for any of $1 \leq j \leq l$, then $\xp_{i_j}$ will
  definitely fall  into the top or bottom $k$ coordinates in the sorted list
  $\xp_{(1)} \leq \dots \leq  \xp_{(d)}$, since all the $d-l \geq  d-k$ remaining
  coordinates in $\vxp$ are bounded by $M$. This means that all the surviving
  coordinates $\xp_{(k+1)}, \dots, \xp_{(d-k)}$ after truncating top and bottom
  $k$ coordinates in $\vxp$ are indeed bounded by $M$ which completes the proof.
\end{proof}

\begin{lem}
\label{lem:x-x'-TSum-k-bound}
  Assume that $\vx = (x_1, \dots, x_d)\in \reals^d$ is given such that $|x_i| \leq M$ for all $1 \leq i \leq d$. Also, assume that $\vxp =
  (\xp_1, \dots, \xp_d) \in \reals^d$ is identical to $\vx$ in all but at most
  $k$ coordinates,  i.e.\
  $\snorm{\vx - \vxp}_0 \leq k$.
  Then, we have
  \begin{equation*}
    |\tsum_k(\vx) - \tsum_k(\vxp)| \leq 6kM.
  \end{equation*}
\end{lem}

\begin{proof}
  Let $x_{\sigma(1)} \leq \dots \leq x_{\sigma(d)}$ and $\xp_{\pp{\sigma}(1)}
  \leq \dots \leq \xp_{\pp{\sigma}(d)}$ be the sorted elements in $\vx$ and
  $\vxp$ with permutations $\sigma$ and $\pp{\sigma}$, respectively. Following
  the definition, we have
  \begin{equation*}
    \tsum_k(\vx) = \sum_{i=k+1}^{d-k} x_{\sigma(i)} = \sum_{i: \sigma^{-1}(i) \in \{k+1, \dots, d-k\}}x_i  = \sum_{i=1}^d \one{\sigma^{-1}(i) \in \{k+1, \dots, d-k\}} x_i.
  \end{equation*}
  Similarly, we have
  \begin{equation*}
    \tsum_k(\vxp) = \sum_{i=1}^d \one{{\pp{\sigma}}^{-1}(i) \in \{k+1, \dots, d-k\}} \pp{x}_i.
  \end{equation*}
  To simplify the notation, for $1 \leq i \leq d$, we define $y_i := \one{\sigma^{-1}(i) \in \{k+1,
    \dots, d-k\}} x_i$ and $\pp{y}_i  := \one{{\pp{\sigma}}^{-1}(i) \in \{k+1,
    \dots, d-k\}} \pp{x}_i$. Moreover, let
  \begin{align*}
    A_1 &:= \{1 \leq i \leq d: \sigma^{-1}(i) \in \{k+1, \dots, d-k\} \text{ and } \pp{\sigma}^{-1}(i) \notin \{k+1, \dots, d-k\}\} \\
    A_2 &:= \{1 \leq i \leq d: \sigma^{-1}(i) \notin \{k+1, \dots, d-k\} \text{ and } \pp{\sigma}^{-1}(i) \in \{k+1, \dots, d-k\}\} \\
    A_3 &:= \{1 \leq i \leq d: \sigma^{-1}(i) \in \{k+1, \dots, d-k\} \text{ and } \pp{\sigma}^{-1}(i) \in \{k+1, \dots, d-k\} \text { and } x_i \neq \pp{x}_i\} \\
    A &:= A_1 \cup A_2 \cup A_3.
  \end{align*}
  Note that if $i \notin A$, either $\sigma^{-1}(i) \notin \{k+1, \dots, d-k\}$ and
  $\pp{\sigma}^{-1}(i) \notin \{k+1, \dots, d-k\}$, in which case $y_i = \pp{y}_i =
  0$; or $\sigma^{-1}(i) \in \{k+1, \dots, d-k\}$, $\pp{\sigma}^{-1}(i) \in \{k+1,
  \dots, d-k\}$, and $x_i = \pp{x}_i$, in which case $y_i = \pp{y}_i = x_i =
  \pp{x}_i$. This means that $y_i = \pp{y}_i$ for $i \notin A$ and
  \begin{equation}
    \label{eq:tsum-diff-A1-A2-A3}
    \begin{aligned}
      |\tsum_k(\vx) - \tsum_k(\vxp)| &\leq \sum_{i \in A} |y_i - \pp{y}_i| \\
      &\leq \sum_{i\in A_1} |y_i - \pp{y}_i| + \sum_{i\in A_2} |y_i - \pp{y}_i|  + \sum_{i\in A_3} |y_i - \pp{y}_i|.
    \end{aligned}
  \end{equation}
  Note that for $i \in A_1$, we have $\pp{y}_i = 0$ and $y_i = x_i$, implying
  $|y_i - \pp{y}_i| = |x_i| \leq M$. On the other hand, for $i \in A_2$, $y_i =
  0$ and $\pp{y}_i = \pp{x}_i$. But since $\pp{\sigma}^{-1}(i) \in \{k+1, \dots,
  d-k\}$, using Lemma~\ref{lem:k-modify-survivals-are-bounded}, we have $|y_i -
  \pp{y}_i| = |\pp{x}_i| \leq  M$. Moreover, for $i \in A_3$, we have $y_i =
  x_i$ and $\pp{y}_i = \pp{x}_i$. Also, from
  Lemma~\ref{lem:k-modify-survivals-are-bounded}, we have $|\pp{x}_i| \leq M$.
  Thereby, $|y_i - \pp{y}_i| \leq |x_i| + |\pp{x}_i| \leq 2M$. Putting all these
  together, we get
  \begin{equation}
    \label{eq:sum-A1-A2-A3-M-bound}
    \sum_{i\in A_1} |y_i - \pp{y}_i| + \sum_{i\in A_2} |y_i - \pp{y}_i|  + \sum_{i\in A_3} |y_i - \pp{y}_i|  \leq M |A_1| + M|A_2| + 2M |A_3|.
  \end{equation}
  Observe that
  \begin{equation}
    \label{eq:A1-bound}
    |A_1| \leq |\{1 \leq i \leq d: \pp{\sigma}^{-1}(i) \notin \{k+1, \dots, d-k\}\}| = 2k.
  \end{equation}
  Similarly,
  \begin{equation}
    \label{eq:A2-bound}
    |A_2| \leq 2k.
  \end{equation}
  On the other hand,
  \begin{equation}
    \label{eq:A3-bound}
    |A_3| \leq |\{1 \leq i \leq d: x_i \neq \pp{x}_i\}| \leq k.
  \end{equation}
  Using~\eqref{eq:A1-bound}, \eqref{eq:A2-bound}, and \eqref{eq:A3-bound} back
  into~\eqref{eq:sum-A1-A2-A3-M-bound} and comparing with
  \eqref{eq:tsum-diff-A1-A2-A3}, we realize that
  \begin{equation*}
    |\tsum_k(\vx) - \tsum_k(\vxp)| \leq 6kM,
  \end{equation*}
  which completes the proof.
\end{proof}

The following is a direct consequence of  Lemma~\ref{lem:x-x'-TSum-k-bound}.

\begin{cor}
\label{cor:vx-vxp-tsum-diff-norm-infty}
  Given $\vx, \vxp \in \reals^d$ and integer $k$ satisfying $\snorm{\vx - \vxp}_0 \leq k < d/2$, we
  have
  \begin{equation*}
    |\tsum_k(\vx) - \tsum_k(\vxp) | \leq 6k \min \{\snorm{\vx}_\infty, \snorm{\vxp}_\infty\}.
  \end{equation*}
\end{cor}


We are now ready to give the proof of Lemma~\ref{lem:trun-ip-bound}:

\begin{proof}[Proof of Lemma~\ref{lem:trun-ip-bound}]
  We have
  \begin{align*}
    |\langle  \vw, \vxp \rangle_k - \langle  \vw, \vx \rangle | &\leq  |\langle  \vw, \vxp \rangle_k - \langle  \vw, \vx \rangle_k | + |\langle  \vw, \vx \rangle_k - \langle   \vw, \vx \rangle| \\
                                                              &\leq |\langle  \vw, \vxp \rangle_k - \langle  \vw, \vx \rangle_k | + 2k \snorm {\vw \odot \vx}_\infty \\
    &= |\tsum_k(\vw \odot \vxp) - \tsum_k(\vw \odot \vx) | + 2k \snorm {\vw \odot \vx}_\infty \\
                                                                &\stackrel{(a)}{\leq} 6k \snorm{\vw \odot \vx}_\infty + 2k \snorm{\vw \odot \vx}_\infty \\
    &= 8k \snorm{\vw \odot \vx}_\infty,
  \end{align*}
  where in step $(a)$ we have used  $\snorm{\vw \odot \vxp - \vw
    \odot \vx}_0 \leq \snorm{\vxp- \vx}_0 \leq k$ together with
  Corollary~\ref{cor:vx-vxp-tsum-diff-norm-infty}. This completes the proof.
\end{proof}

\editfinish


\editstart

\section{Proof of the Upper Bound (Theorem~\ref{thm:upper-bound})}
\label{sec:app-upper-bound-proof}

Given $\vx \in \reals^d$ and $y \in \{\pm 1\}$, define
\begin{equation*}
  \ell^{(k)}(\mC_F^{(k)}; \vx, y) := \max_{\vxp \in \mB_0(\vx,k)} \ell(\mC^{(k)}_F; \vxp, y).
\end{equation*}
We have 
  \begin{align*}
    \ell^{(k)}(\mC_F^{(k)}; \vx, 1)    &= \one{\exists \vxp \in \mB_0(\vx, k): \mC_F^{(k)}(\vxp) \neq 1} \\
                                    &= \one{\exists \vxp\in \mB_0(\vx,k):\langle  \vw(F), \vxp_{F} \rangle_k \leq 0}
  \end{align*}
  Using Lemma~\ref{lem:trun-ip-bound}, for $\vxp$ such that $\snorm{\vxp - \vx}_0 \leq 0$, since
  $\snorm{\vxp_{F} - \vx_{F}}_0 \leq \snorm{\vxp - \vx}_0 \leq k$, we have
  \begin{equation*}
    |\langle  \vw(F), \vxp_{F} \rangle_k - \langle \vw(F), \vx_{F} \rangle| \leq 8k \snorm{\vw(F) \odot \vx_{F}}_\infty.
  \end{equation*}
  This means that
  \begin{equation*}
    \one{\exists \vxp \in \mB_0(\vx,k) : \langle  \vw(F), \vxp_{F} \rangle_k \leq 0} \leq \one{\langle \vw(F) , \vx_{F} \rangle \leq 8k \snorm{\vw(F)\odot \vx_{F}}_\infty},
  \end{equation*}
  and
  \begin{equation}
\label{eq:ev-ellk-1-bound-1}
    \evwrt{(\vx, y) \sim \mD}{\ell^{(k)}(\mC_F^{(k)}; \vx, 1) | y=1} \leq \pr{\langle \vw(F), \vx_{F}  \rangle \leq 8k \snorm{\vw(F)\odot \vx_{F}}_\infty |\, y= 1}.
  \end{equation}
  Let $\Sigma_{F}$ be as defined in~\eqref{eq:SigmaF}
  and let $\tSigma_{F}$ be the diagonal part of $\Sigma_{F}$.
  Note that since $\Sigma$ is positive definite, 
  $\tSigma_{F}$ is diagonal with positive diagonal entries. Hence, we may write
  \begin{equation}
    \label{eq:snorm-w-Fc-X-Fc-infty-bound-1}
    \snorm{\vw(F) \odot \vx_{F}}_\infty = \snorm{(\tSigma^{1/2} \vw(F)) \odot (\tSigma^{-1/2} \vx_{F})}_\infty \leq \snorm{\tSigma_{F}^{1/2} \vw(F)}_\infty \snorm{\tSigma_{F}^{-1/2} \vx_{F}}_\infty.
  \end{equation}
Let $\sigma_i^2$ denote the $i$th diagonal coordinate of $\Sigma$. Fix $i \in F$
and note that conditioned on $y=1$, we have $x_i \sim \mN(\mu_i, \sigma_i^2)$.
On the other hand, with $\va := \tSigma_F^{-1/2} \vx_F$, we have $a_i \sim
\mN(\sigma_i^{-1} \mu_i, 1)$. Note that $\bar{\Phi}(\sigma_i^{-1}\mu_i)$ is the
optimal Bayes classification error of $y$ given $x_i$ only, which is indeed not
smaller than the optimal Bayes classification error of $y$ given the whole
vector $\vx$, which is in turn equal to $\bar{\Phi}(\snorm{\vnu}_2) = \bar{\Phi}(1)$. Since
$\bar{\Phi}$ is decreasing, this implies $\sigma_i^{-1} \mu_i \leq 1$.
Consequently, by union bound, we have
\begin{align*}
  \pr{\snorm{\tSigma_F^{-1/2} \vx_F}_\infty > 1 + \sqrt{2\log d}} &\leq \sum_{i \in F} \pr{a_i - \sigma_i^{-1} \mu_i > \sqrt{2 \log d}} \\
                                                                  &\leq d \bar{\Phi}(\sqrt{2 \log d}) \\
                                                                  &\leq d \frac{1}{\sqrt{2 \pi } \sqrt{2 \log }} e^{-\log d} \\
  &\leq \frac{1}{\sqrt{2 \log d}}. 
\end{align*}
Thereby, we get
\begin{equation}                                                                                                                      %
\label{eq:pr-snorm-tSigma-X-logd-1}                                                                                                   %
  \pr{\snorm{\tSigma_{F}^{-1/2} \vx_{F}}_\infty > 2 \sqrt{2\log d}\, |\, y=1} \leq \frac{1}{\sqrt{2 \log d}}.                          %
\end{equation}                                                                                                                        %
On the other hand, we have
\begin{equation}
\label{eq:snorm-tsigma-w-bound-2}
  \snorm{\tSigma_{F}^{1/2} \vw(F)}_\infty = \snorm{\tSigma^{1/2}_{F} \Sigma_{F}^{-1/2} \vnu(F)}_\infty \leq \snorm{\tSigma^{1/2}_{F} \Sigma_{F}^{-1/2}}_\infty \snorm{\vnu(F)}_\infty,
\end{equation}
where $\snorm{\tSigma^{1/2}_{F} \Sigma_{F}^{-1/2}}_\infty$ denotes the operator
norm of $\tSigma^{1/2}_{F} \Sigma_{F}^{-1/2}$ induced by the vector
$\ell_\infty$ norm.
Using~\eqref{eq:snorm-w-Fc-X-Fc-infty-bound-1},
\eqref{eq:pr-snorm-tSigma-X-logd-1}, and~\eqref{eq:snorm-tsigma-w-bound-2} back
into~\eqref{eq:ev-ellk-1-bound-1} and simplifying, we get
\begin{align*}
  &\evwrt{(\vx,y) \sim \mD}{\ell^{(k)}(\mC_F^{(k)}; \vx, 1) | y= 1} \\
  &\quad\leq \frac{1}{\sqrt{2\log d}} + \pr{\langle \vw(F), \vx_{F}  \rangle \leq 16 k \sqrt{2\log d}  \snorm{\tSigma^{1/2}_{F} \Sigma_{F}^{-1/2}}_\infty \snorm{\vnu(F)}_\infty |\, y = 1} 
\end{align*}
It is easy to see that conditioned on $y=1$, $\langle \vw(F), \vx_F \rangle \sim
\mN(\snorm{\vnu(F)}_2^2, \snorm{\vnu(F)}_2^2)$. Using this in the above bound,
we get
\begin{align*}
  &\evwrt{(\vx,y) \sim \mD}{\ell^{(k)}(\mC_F^{(k)}; \vx, 1) | y= 1} \\
    &\quad\leq \frac{1}{\sqrt{2 \log d }} + \bar{\Phi}\left( \snorm{\vnu(F)}_2 - \frac{16 k \sqrt{2\log d} \snorm{\tSigma_{F}^{1/2} \Sigma_{F}^{-1/2}}_\infty \snorm{\vnu(F)}_\infty}{\snorm{\vnu(F)}_2} \right).
\end{align*}
Due to the symmetry, we have the same bound conditioned on $y=-1$ which yields
the desired result. 

\editfinish


\section{Proof of the Lower Bound in the Diagonal Regime (Theorem~\ref{thm:lower-bound-diag})}
\label{sec:app-lower-diagonal}

Before giving the proof of Theorem~\ref{thm:lower-bound-diag}, we need the
following lemma.

\begin{lem}
\label{lem:error-lower-bound-adv-strategy}
For any random adversarial strategy with budget $k$ which has a density function $f_{\vxp|\vx, y}$, we have
\begin{equation*}
  \optloss_{\vmu,\Sigma}(k) \geq \frac{1}{2} \pr{f_{\vxp|y}(\vxp |1) = f_{\vxp|y}(\vxp|-1)} + \pr{f_{\vxp|y}(\vxp|-1) > f_{\vxp|y}(\vxp|1) \bigg| y = 1},
\end{equation*}
\end{lem}

\begin{proof}
  Note that the right hand side is indeed  the Bayes optimal error associated
  with the MAP estimator assuming that the classifier knows adversary's
  strategy. Since the classifier does not know the adversary's strategy in
  general, the right hand side is indeed a lower bound on the optimal robust
  classification error.
\end{proof}

Now we are ready to prove Theorem~\ref{thm:lower-bound-diag}.

\begin{proof}[Proof of Theorem~\ref{thm:lower-bound-diag}]
Note that when $A$ is empty, there is no adversarial modification and the
standard Bayes analysis implies that 
$\optloss_{\vmu, \Sigma}(0) = \bar{\Phi}(\snorm{\vnu}_2) =
\bar{\Phi}(\snorm{\vnu_{A^c}}_2)$ and the desired bound holds. Hence, we may
assume that $A$ is nonempty for the rest of the proof.

Note that due to~\eqref{eq:vXp-vX-norm-0-log-d-nu-A-1}, the 
  randomized strategy $\adv(A)$ is valid for the adversary given the  budget   $\snorm{\vnu}_1 \log d$.
  Thereby  we may use Lemma~\ref{lem:error-lower-bound-adv-strategy} with $\adv(A)$ to bound
  $\optloss_{\vmu, \Sigma}(\snorm{\vnu_A}_1 \log d)$ from below. Before that, we
  show that with high probability under the above randomized strategy for the
  adversary, recalling the definition of random variables $I_i$ for $i \in A$ from
  \eqref{eq:adv-Zi-Ii-def}, we have $\sum_{i \in A} I_i \leq \snorm{\vnu_A}_1 \log d$ and hence
  $\vxp = \vZ$. It is easy to see that for each $i$, $\pr{I_i = 1| y = 1} =
  \pr{I_i = 1 | y = -1}$; therefore,
  \begin{align*}
    \pr{I_i = 1} = \pr{I_i = 1 | y = \sgn(\mu_i)} &= \int_{0}^\infty [1 - p_i(t, \sgn(\mu_i))] f_{x_i | y}(t|\sgn(\mu_i)) d t \\
                                                  &= \int_0^\infty \left[ 1 - \frac{\exp(-(t+|\mu_i|)^2 / 2\sigma_i^2)}{\exp(-(t-|\mu_i|)^2 / 2\sigma_i^2)} \right]\exp\left( -(t-|\mu_i|)^2/2\sigma_i^2 \right) d t \\
                                                  &= 1 - \bar{\Phi}(|\nu_i|) \\
                                                  &= \text{Erf}(|\nu_i|/\sqrt{2}) \\
    &\leq \left( \sqrt{\frac{2}{\pi}} |\nu_i|\right) \wedge 1.
  \end{align*}
  Hence, we have
  \begin{equation*}
    \pr{I_i = 1} = \pr{I_i = 1 | y = 1} = \pr{I_i = 1 | y = -1} \leq \left( \sqrt{\frac{2}{\pi}} |\nu_i|\right) \wedge 1.
  \end{equation*}
  Therefore, using Markov's inequality, if $I$ is the indicator of the event $\sum_{i \in
    A} I_i > \snorm{\vnu_A}_1 \log d$, we have
  \begin{equation}
    \label{eq:pr-I-1-bound-1-logd}
    \pr{I = 1 } = \pr{I= 1 | y=1} = \pr{I = 1 | y = -1}\leq \frac{\sqrt{2 / \pi} \sum_{i \in A} |\nu_i|}{\snorm{\vnu_A}_1 \log d} \leq \frac{1}{\log d}.
  \end{equation}
Now, we bound
  $\optloss_{\vmu, \Sigma}(\snorm{\vnu_A}_1 \log d)$ from below in 
  the following two cases.

  \underline{Case 1: $A = [d]$}. In this case, using
  Lemma~\ref{lem:error-lower-bound-adv-strategy}, we have
  \begin{align*}
    \optloss_{\vmu, \Sigma}(\snorm{\vnu_A}_1 \log d ) &\geq \frac{1}{2} \pr{f_{\vxp | y}(\vxp | 1) = f_{\vxp|y}(\vxp | -1)} \\
                                                      &\stackrel{(a)}{=} \frac{1}{2} \pr{f_{\vxp | y}(\vxp | 1) = f_{\vxp|y}(\vxp | -1) \,|\, y=1 } \\
                                                      &\geq \frac{1}{2} \pr{f_{\vxp | y}(\vxp | 1) = f_{\vxp|y}(\vxp | -1), I = 0 \,|\, y = 1} \\
                                                      &\stackrel{(b)}{=} \frac{1}{2} \pr{f_{\vZ|y}(\vZ | 1) = f_{\vZ|y}(\vZ|-1) \,|\, y = -1} \\
                                                      &\geq \frac{1}{2} \pr{f_{\vZ|y}(\vZ | 1) = f_{\vZ|y}(\vZ|-1) \,|\, y = 1} - \frac{1}{2} \pr{I = 1\,|\,y = 1}\\
                                                      &\stackrel{(c)}{\geq} \frac{1}{2} - \frac{1}{2 \log d},
  \end{align*}
  where $(a)$ uses the symmetry, $(b)$ uses the fact that when $I = 0$, by
  definition we have $\vxp = \vZ$, and $(c)$ uses~\eqref{eq:adv-Z-cond-f} and~\eqref{eq:pr-I-1-bound-1-logd}.

  \underline{Case 2: $A \subsetneqq [d]$}. Using
  Lemma~\ref{lem:error-lower-bound-adv-strategy}, we have
  \begin{equation}
    \label{eq:adv-lowerb-bound-use-lemma-case-2}
\begin{aligned}
    \optloss_{\vmu, \Sigma}(\snorm{\vnu_A}_1 \log d) &\geq \pr{f_{\vxp| y}(\vxp | -1) > f_{\vxp|y}(\vxp | 1) \,|\, y=1} \\
    &\geq \pr{f_{\vxp| y}(\vxp | -1) > f_{\vxp|y}(\vxp | 1), I = 0 \,|\, y=1} \\
    &\stackrel{(a)}{=}\pr{f_{\vZ| y}(\vZ | -1) > f_{\vZ|y}(\vZ | 1), I = 0 \,|\, y=1} \\
    &\geq \pr{f_{\vZ| y}(\vZ | -1) > f_{\vZ|y}(\vZ | 1)| y=1} - \pr{I = 1\,|\,y = 1} \\
    &\stackrel{(b)}{\geq}  \pr{f_{\vZ| y}(\vZ | -1) > f_{\vZ|y}(\vZ | 1)\,|\, y=1} - \frac{1}{\log d}
\end{aligned}
\end{equation}
where $(a)$ uses the fact that by definition, when $I=0$, we have $\vxp = \vZ$,
and $(b)$ uses~\eqref{eq:pr-I-1-bound-1-logd}.
Note that since $Z_i$ are conditionally independent given $y$, we have
\begin{equation*}
  f_{\vZ|y}(\vZ|y) = f_{\vZ_A | y} (\vZ_A|y) f_{\vZ_{A^c}|y}(\vZ_{A^c}|y).
\end{equation*}
But from~\eqref{eq:adv-Z-cond-f}, we have $f_{\vZ_A|y}(\vZ_A|1) =
f_{\vZ_A|y}(\vZ_A|-1)$ with probability one. Using this
in~\eqref{eq:adv-lowerb-bound-use-lemma-case-2}, we get
\begin{equation*}
  \optloss_{\vmu, \Sigma}(\snorm{\vnu_A}_1 \log d) \geq \pr{f_{\vZ_{A^c}|y}(\vZ_{A^c}|-1) > f_{\vZ_{A^c}|y}(\vZ_{A^c}|1) | y = 1} -\frac{1}{\log d} = \bar{\Phi}(\snorm{\vnu_{A^c}}_2) - \frac{1}{\log d}.
\end{equation*}

We may combine the two cases following the convention that when $A = [d]$, $A^c
= \emptyset$ and $\snorm{\vnu_{A^c}}_2 = 0$. This completes the proof.
\end{proof}


\section{Proof of the General Lower Bound (Theorem~\ref{thm:general-lower-bound})}
\label{sec:app-general-lower}

In this section, we prove Theorem~\ref{thm:general-lower-bound} by providing a general lower bound for the optimal robust classification
error which relaxes the diagonal assumption for the covariance matrix. Our
strategy is to approximate the covariance matrix by a diagonal matrix and use our
lower bound of Theorem~\ref{thm:lower-bound-diag}. It turns out that the optimal
robust classification error is monotone with respect to the positive definite
ordering of the covariance matrix.
Lemma~\ref{lem:optloss-increasing-Sigma-increasing} below formalizes this.
Intuitively speaking, the reason  is that more noise makes the
classification more difficult, resulting in an increase in the optimal robust
classification error.

\begin{lem}
  \label{lem:optloss-increasing-Sigma-increasing}
  Assume that $\vmu \in \reals^d$ and $\Sigma_1$ and $\Sigma_2$ are two positive
  definite covariance matrices such that $\Sigma_1 \preceq \Sigma_2$. Then for
  $0 \leq k \leq d$ we have
  \begin{equation*}
    \optloss_{\vmu, \Sigma_1}(k) \leq \optloss_{\vmu, \Sigma_2}(k). 
  \end{equation*}
\end{lem}

\begin{proof}
Let $y \sim \text{Unif}(\pm 1)$, $\vx_1 \sim \mN(y \vmu, \Sigma_1)$ and $\vx_2
\sim \mN(y \vmu, \Sigma_2)$. 
  Since $\Sigma_1 \preceq \Sigma_2$, we may write $\Sigma_2 = \Sigma_1 + A$ such
  that $A \succeq 0$. In addition to this, we may
  couple $\vx_1, \vx_2$ on the same probability space as $\vx_2 = \vx_1 +\vZ$
  where $\vZ \sim \mN(0,A)$ is independent from all other variables. Now, fix a
  classifier $\mC_2 : \reals^d \rightarrow \{\pm 1\}$ and note that
  \begin{equation}
    \label{eq:l=sigma-2-C-lower-bound-1}
    \begin{aligned}
      \loss_{\vmu, \Sigma_2}(\mC_2, k) &= \pr{\exists \vxp \in \mB_0(\vx_2, k): \mC_2(\vxp) \neq y} \\
      &= \pr{\exists \vxp \in \mB_0(\vx_1 + \vZ, k): \mC_2(\vxp) \neq y} \\
      &= \pr{\exists \vxpp \in \mB_0(\vx_1, k) : \mC_2(\vxpp + \vZ) \neq y} \\
      &\geq \inf_{\tmC_2 : \reals^d \times \reals^d \rightarrow \{\pm 1\}} \pr{\exists \vxpp \in \mB_0(\vx_1, k) : \tmC_2(\vxpp, \vZ) \neq y} \\
    \end{aligned}
  \end{equation}
  Now, fix $\tmC_2 : \reals^d \times \reals^d \rightarrow \{\pm 1\}$ and note
  that using the independence of $Z$, we may write
  \begin{equation}
    \label{eq:tmC2-integeral-bound}
  \begin{aligned}
    \pr{\exists \vxpp \in \mB_0(\vx_1, k) : \tmC_2(\vxpp, \vZ) \neq y} &= \ev{\ev{\one{\exists \vxpp \in \mB_0(\vx_1, k): \tmC_2(\vxpp, \vZ) \neq y} \bigg|\vZ }} \\
                                                                     &= \int \pr{\exists \vxpp \in \mB_0(\vx_1, k): \tmC_2(\vx_1, \vz) \neq y} f_{\vZ}(\vz) d \vz \\
  \end{aligned}
  \end{equation}
  But for $z \in \reals^d$, if we let $\tmC_{2,\vz}(\vx):= \tmC_2(\vx, \vz)$, we get
  \begin{align*}
    \pr{\exists \vxpp \in \mB_0(\vx_1, k): \tmC_2(\vx_1, \vz) \neq y} &= \pr{\exists \vxpp \in \mB_0(\vx_1, k): \tmC_{2,\vz}(\vx_1) \neq y} \\
                                                                    &\geq \inf_{\mC_1 : \reals^d \rightarrow \{\pm 1\}} \pr{\exists \vxpp \in \mB_0(\vx_1, k): \tmC_1(\vx_1) \neq y} \\
    &= \optloss_{\vmu, \Sigma_1}(k).
  \end{align*}
Comparing this with~\eqref{eq:l=sigma-2-C-lower-bound-1}
and~\eqref{eq:tmC2-integeral-bound}, we realize that $\loss_{\vmu,
  \Sigma_2}(\mC_2, k) \geq \optloss_{\vmu, \Sigma_1}(k)$. Since this holds for
arbitrary $\mC_2$, optimizing for $\mC_2$ yields the desired result.
\end{proof}




Note that since $\Sigma$ is positive definite, we have $\Sigma \succeq \alpha
I_d$ where $\alpha>0$ is the minimum eigenvalue of $\Sigma$. Therefore, we may
use Lemma~\ref{lem:optloss-increasing-Sigma-increasing} together with the lower
bound of Theorem~\ref{thm:lower-bound-diag} for $\optloss_{\vmu, \alpha I_d}(.)$
to obtain a lower bound for $\optloss_{\vmu, \Sigma}(.)$. However, it turns out that it is
more efficient in some scenarios to first normalize the diagonal entries of the
covariance matrix. More precisely, define the $d \times d$ matrix $R$ where the
$i,j$ entry in $R$ is $R_{i,j} = \Sigma_{i,j} / \sqrt{\Sigma_{ii} \Sigma_{jj}}$.
In other words, $R_{i,j}$ is the correlation coefficient between the $i$th and
the $j$th coordinates in our Gaussian noise. Equivalently, with $\tSigma$ being
the diagonal part of $\Sigma$, we may write
\begin{equation}
  \label{eq:R-def-dup}
  R := \tSigma^{-\frac{1}{2}} \Sigma \tSigma^{-\frac{1}{2}}.
\end{equation}
It is evident that since $\Sigma$ is assumed to be positive definite, $R$ is
also positive definite. In fact, $R$ is the covariance matrix of the normalized random
vector $\vxp$ such that $\xp_i = x_i / \sqrt{\Sigma}_{i,i}$ where $\vx \sim
\mN(y \vmu, \Sigma)$. Also , all the diagonal entries in $R$ are equal to $1$,
and when $\Sigma$ is diagonal, $R = I_d$ is the identity matrix. Furthermore, we
define $\vu = (u_1, \dots, u_d)$ where
\begin{equation}
  \label{eq:vu-def-dup}
    u_i = \frac{\mu_i}{\sqrt{\Sigma_{i,i}}} \qquad 1 \leq i \leq d.
\end{equation}
In fact, with $\vxp$ being the normalized of $\vx$ as above, we have $\vu =
\ev{\vxp|y=1}$.
In Lemma~\ref{lem;opt-loss-elemenwise-product-invariant}, we show that such coordinate-wise normalization does
not affect the optimal robust classiciation error. The main reason for this is
that any coordinate-wise product of a vector by positive values does not change
the $\ell_0$ norm. This property is unique to the combinatorial $\ell_0$ norm, and
indeed does not hold for $\ell_p$ norms for $p \geq 1$.

\begin{lem}
  \label{lem;opt-loss-elemenwise-product-invariant}
  Given a vector $\va \in \reals^d$ with strictly positive entries, if we define
  $\vmup \in \reals^d$ and $\Sigmap \in \reals^{d \times d}$ as $\mup_i = a_i
  \mu_i$ and $\Sigmap_{i,j} = a_i a_j \Sigma_{i,j}$, then we have
  \begin{equation*}
    \optloss_{\vmu, \Sigma}(k) = \optloss_{\vmup,\Sigmap}(k) \qquad \forall 0 \leq k \leq d.
  \end{equation*}
  In particular, with $\vu$ and $R$ defined above, we have
  \begin{equation*}
    \optloss_{\vmu, \Sigma}(k) = \optloss_{\vu, R}(k) \qquad \forall 0 \leq k \leq d.
  \end{equation*}
\end{lem}

\begin{proof}
  Pick $\epsilon > 0$ together with a  classifier $\mC$ such that
  \begin{equation}
    \label{eq:inv-lem-mC-pick}
    \optloss_{\vmu, \Sigma}(k) \geq \loss_{\vmu, \Sigma}(\mC, k) - \epsilon.
  \end{equation}
  Let $\vx \sim \mN(y \vmu, \Sigma)$, i.e.\ $(\vx, y) \sim \mD$, and define
  $\vxp := \va \odot \vx$. Note that $\vxp \sim \mN(y \vmup, \Sigmap)$. Let
  $\mDp$ denote the joint distribution of $(\vxp, Y)$. Recall that by definition
  $\loss_{\vmu, \Sigma}(\mC, k) = \evwrt{(\vx,y) \sim \mD}{\max_{\vxp \in
      \mB_0(\vx, k)} \ell(\mC; \vxp, y)}$. Note that $\vxp \in \mB_0(\vx, k)$
  iff $\snorm{\vxp - \vx}_0 \leq k$. Since all the entries in $\va$ are
  nonzero, this is equivalent to $\snorm{\va \odot \vxp - \va \odot \vx}_0
  \leq k$ which is in turn equivalent to $\va \odot \vxp \in \mB_0(\va \odot
  \vx, k)$. Therefore, if $\va^{-1}$ denotes the elementwise inverse of $\va$,
  we may write
  \begin{equation*}
    \loss_{\vmu, \Sigma}(\mC, k) = \evwrt{(\vx, y) \sim \mD}{\max_{\vxpp \in \mB_0(\va \odot \vx, k)} \ell(\mC; \va^{-1} \odot \vxpp, y)}.
  \end{equation*}
  Let $\mCp$ be the classifier defined that $\mCp(\vx) := \mC(\va \odot \vx)$.
  With this, we can rewrite the above as
  \begin{align*}
    \loss_{\vmu, \Sigma}(\mC, k) &= \evwrt{(\vx, y) \sim \mD}{\max_{\vxpp \in \mB_0(\va \odot \vx, k)} \ell(\mCp;  \vxpp, y)} \\
                                 &= \evwrt{(\vxp, y) \sim \mDp}{\max_{\vxpp \in \mB_0(\vxp, k)} \ell(\mCp;  \vxpp, y)} \\
                                 &= \loss_{\vmup, \Sigmap}(\mCp, k) \\
    &\geq \optloss_{\vmup, \Sigmap}(k).
  \end{align*}
  Comparing this with~\eqref{eq:inv-lem-mC-pick} and sending to zero, we realize
  that $\optloss_{\vmu, \Sigma}(k) \geq \optloss_{\vmup, \Sigmap}(k)$. Changing
  the order of $(\vmu, \Sigma)$ and $(\vmup, \Sigmap)$ and replacing $\va$ with
  $\va^{-1}$ yields the other direction and completes the proof.
\end{proof}

Using the above tools, we are now ready to prove Theorem~\ref{thm:general-lower-bound}.


\begin{proof}[Proof of Theorem~\ref{thm:general-lower-bound}]
  Note that since $\Sigma$ is positive definite, $R$ is also positive definite
  and $\zeta_\text{min} > 0$. Moreover, we have $R \succeq
  \zeta_\text{min} I_d$. Therefore, using
  Lemmas~\ref{lem:optloss-increasing-Sigma-increasing} and
  \ref{lem;opt-loss-elemenwise-product-invariant} above, we realize that for all
  $k$, we have
  \begin{equation}
    \label{eq:optloss-mu-u-zeta-min}
    \optloss_{\vmu, \Sigma}(k) = \optloss_{\vu, R}(k) \geq \optloss_{\vu, \zeta_\text{min} I_d}(k).
  \end{equation}
  Since $\zeta_\text{min}I_d$ is diagonal, we may use our lower bound of
  Theorem~\ref{thm:lower-bound-diag} with $\vnu = (\zeta_\text{min}I_d)^{-1/2}
  \vu = \vu / \sqrt{\zeta_\text{min}}$ to obtain the following bound with holds
  for all $A \subseteq [d]$
  \begin{equation*}
    \optloss_{\vu, \zeta_\text{min}I_d}\left( \frac{1}{\sqrt{\zeta_\text{min}}} \snorm{\vu_A}_1 \log d \right) \geq \bar{\Phi}(\snorm{\vu_{A^c}}_2) - \frac{1}{\log d}.
  \end{equation*}
  The proof is complete by comparing this with~\eqref{eq:optloss-mu-u-zeta-min}.
\end{proof}


\section{Proof of Theorem~\ref{thm:ind-bound-finite}}
\label{sec:app_diagonal_match}

  We use the bound in Corollary~\ref{cor:upper-bound-diagonal} with $F =
[\lambda_c:d]$, which  simplifies 
into the following with $k =
\snorm{\vnu_{[1:\lambda_c]}}_1 / \log d$: 
\begin{equation}
  \label{eq:ind-thm-ach-bound-simple}
  \optloss_{\vmu, \Sigma}\left( \frac{\snorm{\vnu_{[1:\lambda_c]}}_1}{\log d} \right) \leq \frac{1}{\sqrt{2 \log d}} + \bar{\Phi}\left( \snorm{\vnu_{[\lambda_c:d]}}_2 - \frac{\snorm{\vnu_{[1:\lambda_c]}}_1 \snorm{\vnu_{[\lambda_c: d]}}_\infty}{\snorm{\vnu_{[\lambda_c  : d]}}_2} \frac{16\sqrt{2}}{\sqrt{\log d}}\right).
\end{equation}
Note that we have
\begin{equation}
\label{eq;norm-2-vnu-lambda-c-d-bound}
  \snorm{\vnu_{[\lambda_c:d]}}_2^2  = 1- \snorm{\vnu_{1:\lambda_c-1}}_2^2 \geq 1 - c^2.
\end{equation}
On the other hand,
\begin{equation} 
  \label{eq:bound-lamnbda-c-norm-1-inf-bound}
  \begin{aligned}
    \snorm{\vnu_{[1:\lambda_c]}}_1 \snorm{\vnu_{[\lambda_c: d]}}_\infty & = \snorm{\vnu_{[1:\lambda_c]}}_1 |\nu_{\lambda_c}| \\
    &\leq \snorm{\vnu_{[1:\lambda_c]}}_2^2 \\
    &\leq \snorm{\vnu}_2^2 \\
    &= 1
  \end{aligned}
\end{equation}
Substituting~\eqref{eq;norm-2-vnu-lambda-c-d-bound}
and~\eqref{eq:bound-lamnbda-c-norm-1-inf-bound} back
into~\eqref{eq:ind-thm-ach-bound-simple}, we get
\begin{equation}
  \label{eq:ind-ach-bound-final-simple}
    \optloss_{\vmu, \Sigma}\left( \frac{\snorm{\vnu_{[1:\lambda_c]}}_1}{\log d} \right) \leq \frac{1}{\sqrt{2 \log d}} + \bar{\Phi}\left( \sqrt{1-c^2} - \frac{16\sqrt{2}}{\sqrt{1-c^2} \sqrt{\log d}} \right )
\end{equation}
Furthermore, with $A =[1:\lambda_c]$, the bound in
Theorem~\ref{thm:lower-bound-diag} implies that
\begin{equation}
  \label{eq:ind-conv-bound-simple}
  \optloss_{\vmu, \Sigma}(\snorm{\vnu_{[1:\lambda_c]}}_1 \log d) \geq \bar{\Phi}(\sqrt{1-c^2}) - \frac{1}{\log d}.
\end{equation}
This completes the proof.



\section{Proof of Theorem~\ref{thm:asymp-independent}}
\label{app:asymp-theorem-proof}

  Note that since $\funcd(.)$ is nondecreasing for all $d$, if $\funcinf(c) =
  \lim \funcd(c)$  exists, $\funcinf(.)$ is indeed nondecreasing and
  $\funcinf(0)$ is well-defined.
  
\underline{Part 1} First we assume that $c \in (0,1)$. Since $\funcinf(c) = \lim
\funcd(c)$ and $\log \log d / \log d \rightarrow 0$, $\limsup \log_d k_d <
\funcinf(c)$ implies that for $d$ large enough, we have
\begin{equation*}
  \log_d k_d < \funcd(c) - \frac{\log \log d}{\log d}.
\end{equation*}
Thereby,
\begin{equation*}
  \log_d k_d < \log_d \snorm{\vnud_{[1:\lambdad_c]}}_1 - \frac{\log \log d}{\log d} = \log_d \frac{ \snorm{\vnud_{[1:\lambdad_c]}}_1}{\log d}.
\end{equation*}
Hence, Theorem~\ref{thm:ind-bound-finite} implies that
\begin{equation*}
  \optloss_d(k_d) \leq \optloss_d\left( \frac{ \snorm{\vnud_{[1:\lambdad_c]}}_1}{\log d} \right) \leq \frac{1}{\sqrt{2 \log d}} + \bar{\Phi}\left( \sqrt{1-c^2} - \frac{16\sqrt{2}}{\sqrt{1-c^2} \sqrt{\log d}} \right).
\end{equation*}
Sending $d$ to infinity, we get $\limsup \optloss_d(k_d) \leq
\bar{\Phi}(\sqrt{1-c^2})$. Next, we consider $c =0$. Note that since
$\funcinf(.)$ is nondecreasing, $\limsup \log_d k_d < \funcinf(0)$ implies that
$\limsup \log_d k_d < \funcinf(c)$ for all $c > 0$. Consequently, the above
bound implies that $\limsup \optloss(k_d) \leq \bar{\Phi}(\sqrt{1-c^2 })$ for all
$c > 0$. Sending $c$ to zero, we realize that $\limsup \optloss(k_d) \leq
\bar{\Phi}(0)$. Finally, for $c = 1$, note that the classifier that always
outputs $1$ has misclassification error at most $1/2$. This implies that
irrespective of the sequence $k_d$, we always have $\limsup \optloss_d(k_d) \leq
1/2 = \bar{\Phi}(\sqrt{1-1^2})$ and the bound automatically holds for $c = 1$.

\underline{Part 2} First we assume that $c \in (0,1]$. Similar to the first
pare, $\liminf \log_d k_d > \funcinf(c)$ implies that for $d$ large enough, we
have
\begin{equation*}
  \log_d k_d > \funcd(c) + \frac{\log \log d}{\log d},
\end{equation*}
and
\begin{equation*}
    \log_d k_d > \log_d \snorm{\vnud_{[1:\lambdad_c]}}_1 + \frac{\log \log d}{\log d} = \log_d  (\log d\snorm{\vnud_{[1:\lambdad_c]}}_1).
  \end{equation*}
  Hence, Theorem~\ref{thm:ind-bound-finite} implies that
  \begin{equation*}
    \optloss_d(k_d) \geq \optloss_d(\log d \snorm{\vnud_{[1:\lambdad_c]}}_1) \geq \bar{\Phi}(\sqrt{1-c^2}) - \frac{1}{\log d}. 
  \end{equation*}
  Sending $d \rightarrow \infty$, we get $\liminf \optloss_d(k_d) \geq
  \bar{\Phi}(\sqrt{1-c^2})$. For the case $c = 0$, note that irrespective of the
  sequence $k_d$, we always have $\optloss_d(k_d) \geq \optloss_d(0) =
  \bar{\Phi}(\sqrt{1-0^2})$. Thereby, the result for $c =0$ automatically holds. 




\end{document}